\documentclass{article}

\PassOptionsToPackage{numbers, compress}{natbib}
\usepackage[preprint]{neurips_2024}

\usepackage[utf8]{inputenc} % allow utf-8 input
\usepackage[T1]{fontenc}    % use 8-bit T1 fonts
\usepackage{hyperref}       % hyperlinks
\usepackage{url}            % simple URL typesetting
\usepackage{booktabs}       % professional-quality tables
\usepackage{amsfonts}       % blackboard math symbols
\usepackage{nicefrac}       % compact symbols for 1/2, etc.
\usepackage{microtype}      % microtypography
\usepackage{xcolor}         % colors
\hypersetup{
    colorlinks=true,  
    linkcolor=blue,  
    filecolor=magenta, 
    urlcolor=cyan,  
    citecolor=goldenrod
}
\definecolor{goldenrod}{RGB}{218, 165, 32}
\usepackage{subfigure}
\usepackage{times}
\usepackage{soul}
\usepackage[utf8]{inputenc}
\usepackage[small]{caption}
\usepackage{graphicx}
\usepackage{amsmath}
\usepackage{amsthm}
\usepackage{booktabs}
\usepackage{algorithm}
\usepackage{algorithmic}

\newtheorem{proposition}{Proposition}
\newtheorem{definition}{Definition}
\newtheorem{corollary}{Corollary}

\usepackage{amsfonts,amssymb}
\usepackage{longtable}

\usepackage{multirow}
\usepackage{enumitem}
\usepackage{color}
\usepackage{wrapfig}
\definecolor{darkgreen}{rgb}{0.0, 0.5, 0.0}
\definecolor{peach}{rgb}{1.0, 0.85, 0.7}
\definecolor{lightred}{rgb}{1.0, 0.6, 0.6}
\definecolor{deepred}{rgb}{0.6, 0.0, 0.0}

\usepackage[textsize=tiny]{todonotes}

\title{Enhancing Graph Transformers with Hierarchical Distance Structural Encoding}

\author{%
  Yuankai Luo\(^{1,3}\) \\
  \And
  Hongkang Li\(^2\) 
  \And
  Lei Shi\(^1\)\thanks{Corresponding authors.}
  \And
  Xiao-Ming Wu\(^3\)\footnotemark[1] \\
  \AND
  \textnormal{\(^1\)Beihang University} \And
  \textnormal{\(^2\)Rensselaer Polytechnic Institute} 
  \AND
  \textnormal{\(^3\)The Hong Kong Polytechnic University} 
  \AND 
  \textnormal{\texttt{luoyk@buaa.edu.cn}} \And
  \textnormal{\texttt{xiao-ming.wu@polyu.edu.hk}}
}

\begin{document}

\maketitle

\begin{abstract}
  
Graph transformers need strong inductive biases to derive meaningful attention scores. Yet, current methods often fall short in capturing longer ranges, hierarchical structures, or community structures, which are common in various graphs such as molecules, social networks, and citation networks. This paper presents a Hierarchical Distance Structural Encoding (HDSE) method to model node distances in a graph, focusing on its multi-level, hierarchical nature. We introduce a novel framework to seamlessly integrate HDSE into the attention mechanism of existing graph transformers, allowing for simultaneous application with other positional encodings. To apply graph transformers with HDSE to large-scale graphs, we further propose a high-level HDSE that effectively biases the linear transformers towards graph hierarchies. We theoretically prove the superiority of HDSE over shortest path distances in terms of expressivity and generalization. Empirically, we demonstrate that graph transformers with HDSE excel in graph classification, regression on 7 graph-level datasets, and node classification on 11 large-scale graphs, including those with up to a billion nodes. Our
implementation is available at \url{https://github.com/LUOyk1999/HDSE}.
\end{abstract}

\vspace{-0.2 in}
\section{Introduction}

The success of Transformers \cite{vaswani2017attention} in various domains, including natural language processing (NLP) and computer vision \cite{dosovitskiy2020image}, has sparked significant interest in developing transformers for graph data~\cite{dwivedi2020generalization,ying2021transformers,kreuzer2021rethinking,chen2022structure,rampavsek2022recipe,ma2023graph,zhang2023rethinking,wu2023simplifying}. 
Scholars have turned their attention to this area, aiming to address the limitations of Message-Passing Graph Neural Networks (MPNNs) \cite{gilmer2017neural} such as over-smoothing \cite{li2018learning} and over-squashing~\cite{alon2020bottleneck, topping2021understanding}. 
% Specially, researchers have explored the adaptability of the Transformer architecture to graph machine learning tasks at both the node-level and graph-level, building upon its powerful self-attention mechanism. Despite the inherent limitations of traditional MPNNs, the promising performance of graph transformer models has motivated further exploration and expansion of their application in various graph machine learning scenarios.

% However, Transformers are known to have a general lack of strong inductive biases \cite{dosovitskiy2020image}. In contrast to MPNNs, graph transformers do not rely on fixed graph structure information. Instead, they calculate pairwise interactions for all nodes within a graph and represent positional and structural data using soft inductive biases. The mechanism is inherently flat and does not learn hierarchical representations of graphs. Furthermore, the time and space complexity of the global all-pair attention mechanism increases quadratically with the number of nodes.

However, Transformers \cite{vaswani2017attention} are known for their lack of strong inductive biases \cite{dosovitskiy2020image}. In contrast to MPNNs, graph transformers do not rely on fixed graph structure information. Instead, they compute pairwise interactions for all nodes within a graph and represent positional and structural data using more flexible, soft inductive biases. Despite its potential, this mechanism does not have the capability to learn hierarchical structures within graphs. Developing effective positional encodings is also challenging, as it requires identifying important hierarchical structures among nodes, which differ significantly from other Euclidean domains~\cite{bronstein2021geometric}. Consequently, graph transformers are prone to overfitting and often underperform MPNNs when data is limited~\cite{ma2023graph}, especially in tasks involving large graphs with relatively few labeled nodes~\cite{wu2023simplifying}. These challenges become even more significant when dealing with various molecular graphs, such as those found in polymers or proteins. These graphs are characterized by a multitude of  substructures and exhibit long-range and hierarchical structures. The inability of graph transformers to learn these hierarchical structures can significantly impede their performance in tasks involving such complex molecular graphs.

Further, the global all-pair attention mechanism in transformers poses a significant challenge due to its time and space complexity, which increases quadratically with the number of nodes. This quadratic complexity significantly restricts the application of graph transformers to large graphs, as training them on graphs with millions of nodes can require substantial computational resources. Large-scale graphs, such as social networks and citation networks, often exhibit community structures characterized by closely interconnected groups with distinct hierarchical properties. To enhance the scalability and effectiveness of graph transformers, it is crucial to incorporate hierarchical structural information at various levels.

\begin{figure*}[t]   \center{\includegraphics[width=14.0cm]  {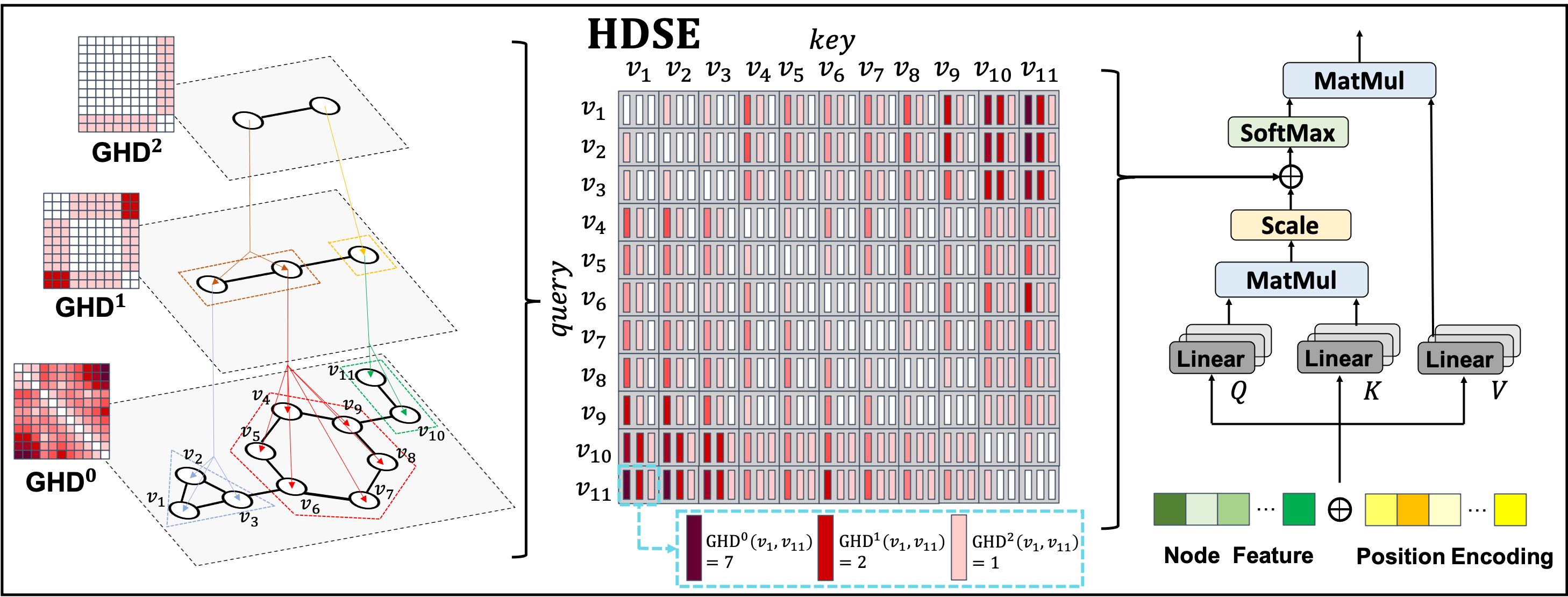}}   \vspace{-0.1 in} \caption{\label{1} 
Overview of our proposed hierarchical distance structural encoding (HDSE) and its integration with graph transformers. HDSE uses the graph hierarchy distance (GHD, refer to Definition \ref{def1}) that can capture interpretable patterns in graph-structured data by using diverse graph coarsening algorithms.
%Overview of our architecture. The Hierarchical Distance Structural Encoding (HDSE) architecture uses the graph hierarchy distance (GHD, refer to Definition \ref{def1}) which enables capturing interpretable patterns in graph-structured data using a variety of coarsening algorithms.
% For the given graph $G$, we initially apply coarsening algorithms to generate the graph hierarchy. Next, we compute the multi-level hierarchy distance and subsequently encode it into HDSE. Finally, we integrate HDSE into the graph transformer. 
Darker colors indicate longer distances.}  \vspace{-0.1 in} \label{fig:architecture} \end{figure*}

To address the aforementioned challenges and unlock the true potential of the transformer architecture in graph learning, we propose a Hierarchical Distance Structural Encoding (HDSE) method (\textbf{Sec.~\ref{subsec:HDSE}}), which can be combined with various graph transformers to produce more expressive node embeddings. HDSE encodes the hierarchy distance, a metric that measures the distance between nodes in a graph, taking into account multi-level graph hierarchical structures. We utilize popular coarsening methods~\cite{karypis1998software,ng2001spectral,girvan2002community,blondel2008fast,loukas2019graph} to construct graph hierarchies, enabling us to measure the distance relationship between nodes across various hierarchical levels.

HDSE enables us to incorporate a robust inductive bias into existing transformers and address the issue of lacking canonical positioning. To achieve this, we introduce {a novel framework} (\textbf{Sec.~\ref{subsec:GT_HDSE}}), as illustrated in Figure \ref{fig:architecture}. We utilize an end-to-end trainable function to encode HDSE as structural bias weights into the attentions, allowing the graph transformer to integrate both HDSE and other positional encodings simultaneously. Our theoretical analysis demonstrates that \emph{{graph transformers equipped with HDSE are significantly more powerful than the ones with the commonly used shortest path distances, in terms of both expressiveness and generalization}}. We rigorously evaluate our HDSE in ablation studies and show that it successfully improves different kinds of baseline transformers, from vanilla graph transformers \cite{dwivedi2020generalization} to state-of-the-art graph transformers \cite{rampavsek2022recipe, chen2022structure, zhang2022hierarchical, ma2023graph}, across 7 graph-level datasets.

% Furthermore, extensive experiments on 7 graph-level tasks show that HDSE effectively enhances various types of baseline transformers.

% This, in turn, allows us to bias existing transformers and alleviate the lack of canonical positioning. 
% To achieve this objective, we introduce a novel framework, illustrated in Figure \ref{fig:architecture}. We use an end-to-end trainable function to learn the multi-level hierarchy distance as bias into attention, that empowers the graph transformer to integrate both HDSE and other positional encodings simultaneously. We conduct extensive experiments on 7 graph-level tasks to demonstrate that our HDSE method successfully enhances various types of baseline transformers. Our theoretical analysis shows that graph transformers equipped with HDSE are strictly stronger than the commonly used shortest path distances (SPD), both in terms of expressivity and generalization.

To enable the application of graph transformers with HDSE to large graphs ranging from millions to billions of nodes, we introduce a high-level HDSE (\textbf{Sec.~\ref{subsec:HGT_HDSE}}), which effectively biases the linear transformers towards the multi-level structural nature of these large networks. 
% The use of hierarchical structures preserves structural information and widens the receptive field, avoiding over-squashing. Meanwhile, it allows for the use of multi-hop features, preventing over-smoothing.
% Utilizing coarsening results, we project the input node features onto a low-dimensional cluster feature space. This enables input nodes to aggregate information from these more compact clusters while being guided by the bias of high-level hierarchy distance.
We demonstrate our high-level HDSE method exhibits high efficiency and quality across 11 large-scale node classification datasets, with sizes up to the billion-node level.

\section{Background and Related Works}\label{sec:related}
%"notations" is no English word in this context, used only for something else
%\subsection{Notation.} 

We refer to a \emph{graph} as a tuple $G =(V, E, \mathbf{X})$, with node set $V$, edge set $E\subseteq V\times V$, and node features $\mathbf{X}\in \mathbb{R}^{|V|\times d}$. Each row in $\mathbf{X}$ represents the feature vector of a node, with $|V|$ denoting the number of nodes and feature dimension $d$. The features of node $v$ are denoted by $x_v\in\mathbb{R}^{d}$.

\subsection{Graph Hierarchies}\label{sec:hie} Given an input graph $G$, a graph hierarchy of $G$ consists of a sequence of graphs $(G^k,\phi_k)_{k\ge 0}$, where $G^0=G$ and  $\phi_k: V^k\to V^{k+1}$ are surjective node mapping functions. 
%That is, $\phi_k$ determines $V^{k+1}$ and
Each node $v^{k+1}_j \in V^{k+1}$ represents a \emph{cluster} of a subset of nodes $\{v^k_i\} \subseteq V^{k}$. This partition can be described by a projection matrix $\hat{P}^{k} \in \{0, 1\}^{{|V^{k}|\times |V^{k+1}|}}$, where $\hat{P}^{k}_{ij}=1$ if and only if $v^{k+1}_j = \phi_{k}(v^k_i)$. The normalized version can be defined by $P^k=\hat{P}^{k}{C^k}^{-{1/2}}$, where $C^k \in \mathbb{R}^{{|V^{k+1}|\times |V^{k+1}|}}$ is a diagonal matrix with its $j$-th diagonal entry being the cluster size of $v^{k+1}_j$. We define the node feature matrix $\mathbf{X}^{k+1}$ for $G^{k+1}$ by $\mathbf{X}^{k+1}={P^k}^{\top}\mathbf{X}^{k}$, where each row of $\mathbf{X}^{k+1}$ represents the average of all entries within a specific \emph{cluster}.
% VT think this is notation just confusing or is this needed $\phi^{-1}_k(v^{k+1}_i) = \{v^k_j \in V^k | \phi_k(v^k_j) = v^{k+1}_i\}$. 
The edge set $E^{k+1}$ of $G^{k+1}$ is defined as $E^{k+1} = \{(u^{k+1}, v^{k+1}) | \exists v^k_r \in \phi_k^{-1}(u^{k+1}), v^k_s \in \phi_k^{-1}(v^{k+1}), \text{ such that } (v^k_r, v^k_s) \in E^k\}$. 

% We construct graph hierarchies by applying well-known graph coarsening algorithms iteratively. For a detailed comparison of complexities, refer to Table~\ref{tab:comparison}. 
% Graph hierarchies can be generated by iteratively applying graph coarsening algorithms. Table~\ref{tab:comparison} summarizes popular graph coarsening algorithms and their complexities.
% Specifically, the algorithm takes a graph $G^k$ and produces $\phi_k: V^k \rightarrow V^{k+1}$.
Graph hierarchies can be constructed by repeatedly applying graph coarsening algorithms. These algorithms take a graph, $G^k$, and generate a mapping function $\phi_k: V^k \rightarrow V^{k+1}$, which maps the nodes in $G^k$ to the nodes in the coarser graph $G^{k+1}$. A summary and comparison of popular graph coarsening algorithms, along with their computational complexities, can be found in Table~\ref{tab:comparison}. We define the \emph{coarsening ratio}  as $\alpha = \frac{|V^{k+1}|}{|V^k|}$, which represents the proportion of the number of nodes in the coarser graph $G^{k+1}$ to the number of nodes in the original graph $G^k$. Consequently, each graph $G^{k}$, where $k>0$, captures specific substructures derived from the preceding graph. 
% Empirically, it is worth noting that we mainly focus on a maximal hierarchy level of $K=1$. This choice aligns with related works \cite{zhang2022hierarchical,zhu2023hierarchical} and has showed good performance in our evaluation.
%We also define the \emph{coarsening ratio} $\alpha = \frac{|V^{k+1}|}{|V^k|}$, this is often a parameter of the algorithm as well.In this way, each graph $G^{k}$, $k>0$, captures certain substructures from the previous graph and the algorithm's rationale decides about the semantics of the clusters. 
%We denote the cluster $v^{k'}$ a node $v^0\in V^0$ is assigned to at level $k'$ by $\phi_{k'-1}...\phi_{0}(v^0)$. 
%Finally, The graph hierarchy $\{G^k,\phi_k\}_{k\ge 0}$ is constructed by recursively applying the algorithms to summarize multi-level graph structures.
% Observe that, empirically, we mainly focus on a maximal hierarchy level of $K=1$. This is in line with related works \cite{zhang2022hierarchical,zhu2023hierarchical} and showed good performance in our evaluation.

\begin{table}
	\centering
        \footnotesize
        % \vspace{-0.1 in}
        \caption{Comparison of popular graph coarsening algorithms.}
	\begin{tabular}{cccccc}
		\toprule
		Coarsening algorithm &METIS \cite{karypis1998software}  & Spectral \cite{ng2001spectral} & Loukas \cite{loukas2019graph} & Newman \cite{girvan2002community} & Louvain \cite{blondel2008fast}\\
        \midrule %
		Complexity & $O(|E|)$  & $O(|V|^3)$   & $O(|V|)$  & $O(|E|^2|V|)$  & $O(|V|\log |V|)$   \\
		\bottomrule
	\end{tabular}
	\label{tab:tab1}\label{tab:comparison}
  \vspace{-0.15 in}
\end{table}

\vspace{-0.05 in}
\subsection{Graph Transformers} 
{Transformers} \citep{vaswani2017attention} have recently 
gained significant attention in graph learning, due to their ability to learn intricate relationships that extend beyond the capabilities of conventional GNNs, and in a unique way. 
The architecture of these models primarily contain a \emph{self-attention} module and a feed-forward neural network, each of which is followed by a residual connection paired with a normalization layer. 
 % \shil{In case of strict page limit, we could drop the following equations for Transformer, which are norms.}
Formally, the self-attention mechanism involves projecting the input node features $\mathbf{X}$ using three distinct matrices $\mathrm{W}_\mathbf{Q}\in \mathbb{R} ^{d\times d'}$, $\mathrm{W}_\mathbf{K}\in \mathbb{R} ^{d\times d'}$ and $\mathrm{W}_\mathbf{V}\in \mathbb{R} ^{d\times d'}$, resulting in the representations for query ($\mathbf{Q}$), key ($\mathbf{K}$), and value ($\mathbf{V}$), which are then used to compute the output features described as follows:
 %
 % \begin{linenomath*}
 \begin{equation*}\label{selfatt1}
 \mathbf{Q}=\mathbf{X}\mathrm{W}_\mathbf{Q}, \ \mathbf{K}=\mathbf{X}\mathrm{W}_\mathbf{K}, \ \mathbf{V}=\mathbf{X}\mathrm{W}_\mathbf{V},
 \end{equation*}
 \begin{equation}
 	\label{eq:transformer} \mathrm{Attention}\left(\mathbf{X}\right) =\mathrm{softmax} \left( \frac{\mathbf{QK}^{\top}}{\sqrt{d'}} \right) \mathbf{V}.
 \end{equation}
 \vspace{-0.15 in}
 % \end{linenomath*}
%

Technically, transformers can be considered as message-passing GNNs operating on fully-connected graphs, decoupled from the input graphs. 
%, % and, as vanilla transformersNote that it %self-attention 
%and the basic model is equivariant to node permutations. %, % of the input nodes, 
 %for nodes with the same features, vanilla transformers generate the same representations, regardless of their connections in the graph. 
%  The graph structure 
%  is usually injected through additions to the original features, %with some information about the locations of nodes via 
% via so-called positional encodings (PEs).
The main research question in the context of graph transformers is how to incorporate the structural bias of the given input graphs % (i.e., which nodes are actually connected) into the transformer architecture 
by adapting the attention mechanism or by augmenting the original features. 
% via so-called positional encodings (PEs).
 The \textbf{Graph Transformer (GT)} \citep{dwivedi2020generalization} represents an early work using Laplacian eigenvectors as positional encoding (PE), and various extensions and alternative models have been proposed since then \citep{min2022transformer}. % cite Transformer for Graphs: An Overview from Architecture Perspective
%For instance, \citep{mialon2021graphit} proposed a relative PE \citep{shaw2018self} by means of kernels on graphs to bias the self-attention calculation.
%Notably, \cite{ying2021transformers} proposed Graphormer and first injected the spatial encoding based on the distance of the shortest path. 
% Next, GraphTrans \citep{wu2021representing} was the first hybrid architecture, using a stack of message-passing GNN layers before the regular transformer layers.  
% [TODO ]
For instance, %have reformulated the self-attention mechanism as a kernel smoother as below and incorporated structure information into their
the \textbf{structure-aware transformer (SAT)}  \citep{chen2022structure} extracts a subgraph representation rooted at each node before computing the attention. 
% : %ing the attention:
% [TODO what do they do with the subgraph]
%  \begin{linenomath*}
% \begin{equation}\label{selfatt3}
% 	\mathrm{Attention}\left( x_v \right) =\sum_{u\in V}{\frac{\kappa \left( x_v,x_u \right)}{\sum_{w\in V}{\kappa \left( x_v,x_w \right)}}}f\left( x_u \right) ,\forall v\in V,
% \end{equation}
%  \end{linenomath*}
% with $f(x)=\mathbf{W}_{\mathbf{V}} x$, and  non-symmetric exp. kernel~$\kappa$:
%  \begin{linenomath*}
%  \label{selfatt4}
% \begin{equation*}
% 	\kappa \left( x,x^{\prime} \right) =\exp \left( \frac{\left. \langle \varphi(x)\mathbf{W}_{\mathbf{Q}},\varphi(x^{\prime})\mathbf{W}_{\mathbf{K}} \right. \rangle}{\sqrt{d_{K\,\,}}} \right),
%  \end{equation*}
%  \begin{equation*}
%  \varphi(x) = \begin{cases}
% x & \text{in vanilla transformer} \\
% \text{GNN}_G(x) & \text{in SAT} \\
% \end{cases}
% \end{equation*}
%  \end{linenomath*}
% where $\langle\cdot, \cdot\rangle$ is the dot product on $\mathbb{R}^d$ and $\text{GNN}_G(x)$ is % a GNN extractor such as 
% an arbitrary GNN model.
Most initial works in the area focus on the classification of smaller graphs, such as molecules; yet, recently, \textbf{GraphGPS} \citep{rampavsek2022recipe} and follow-up works \cite{zhao2021gophormer,wu2022nodeformer,wu2023difformer,wu2023simplifying, chen2022nagphormer, pmlr-v202-kong23a, shirzad2023exphormer, deng2024polynormer, fu2024vcrgraphormer} also consider larger graphs. 
% and focusing on the development of scalable models.
% , and
% Nodeformer \citep{wu2022nodeformer} is designed to address the issue of scalability and expressivity for node classification.
% Altogether, the transformer architecture opens new and promising avenues for graph representation learning, beyond message-passing GNNs. 
% However, to the best of our knowledge, transformers have not been studied particularly in the context of DAGs. 
% In fact, works on graph transformers have dropped their graph-specific PEs in their DAG experiments \citep{chen2022structure, rampavsek2022recipe}, seemingly because they do not provide the right bias in the context of this specific type of graph.

\subsection{Hierarchy in Graph Learning} 
In message passing GNNs, hierarchical pooling of node representations is a proven method for incorporating coarsening into reasoning \cite{bianchi2020spectral, gao2019graph, ying2018hierarchical, lee2019self, huang2019attpool, ranjan2020asap}. With GNNs, coarsened graph representations are further considered in the context of molecules \cite{jin2020hierarchical} and virtual nodes \cite{hwang2022an}. Additionally, HTAK~\cite{bai2022hierarchical} employ graph hierarchies to develop a novel graph kernel by transitively aligning the nodes across multi-level hierarchical graphs. The recent \textbf{HC-GNN} \cite{zhong2023hierarchical} demonstrates competitive performance in node classification on large-scale graphs, utilizing hierarchical community structures for message passing.

%Overall, there are few hierarchical graph transformer models so far. 
In graph transformers, there are currently only a few hierarchical models. \textbf{ANS-GT}~\cite{zhang2022hierarchical} use adaptive node sampling in their graph transformer, enabling it for large-scale graphs and capturing long-range dependencies. Their model groups nodes into super-nodes and allows for interactions between them.
%\citet{zhang2022hierarchical} use adaptive node sampling within graph transformer to enable large-scale graphs and also capture long-range dependencies. Essentially, their \textbf{ANS-GT} groups the original nodes into so-called super-nodes and enables the former to interact with sampled super nodes. 
Similarly, \textbf{HSGT} \cite{zhu2023hierarchical} aggregates multi-level graph information and employs graph hierarchical structure to construct intra-level and inter-level transformer blocks. The intra-level block facilitates the exchange and transformation of information within the local context of each node, while the inter-level block adaptively coalesces every substructure present. Our concurrent work directly incorporates hierarchy into the attention, a fundamental building block of the transformer architecture, making it flexible and applicable to existing graph transformers.
% We note that our, concurrent work incorporates the hierarchy in a more direct way, into the attention, a basic building block of the transformer architecture. Hence it is flexibly applicable on top of existing graph transformers.
% In this framework, the horizontal block facilitates the exchange and transformation of information within the local context of each node, while the vertical block adaptively coalesces every substructure present. 
% \veronika{would remove that, from your text it's not clear what is complex there, they just apply stacked transformers?  - I mean yours is completely dependent on the baseline and does not specifically adress scalability? -, and what's the problem in smaller graphs?: However, these methods present an empirical challenge for graph representation in smaller graphs, primarily due to their complex information aggregating processes. } \luo{I conducted experiments and found that intra-level and inter-level transformer blocks perform poorly on small datasets.}
%In contrast, our HDSE approach is designed to be simple and efficient, making it applicable to both large and small graphs. 
% removing this since it did not get accepted and you assume to not want to compare to it?
Additionally, \textbf{Coarformer} \cite{kuang2021coarformer} utilizes graph coarsening techniques to generate coarse views of the original graph, which are subsequently used as input for the transformer model. 
Likewise, 
PatchGT \cite{gao2022patchgt} starts by segmenting graphs into patches using spectral clustering and then learns from these non-trainable graph patches. \textbf{MGT} \cite{ngo2023multiresolution} learns atomic representations and groups them into meaningful clusters, which are then fed to a transformer encoder to calculate the graph representation. However, these approaches typically yield coarse-level representations that lack comprehensive awareness of the original node-level features~\cite{jiang2023agformer}. In contrast, our model integrates hierarchical information from a broader distance perspective, thereby avoiding the oversimplification in these coarse-level representations. %Evidently, this approach typically yields a coarse-level representation that lacks comprehensive awareness of the original node-level features during the learning process.

\section{Our Method}
%In this section, we introduce the concept of \emph{graph hierarchy distance} (GHD) and detail how it can serve as a faithful inductive bias inside graph transformers.
%a distance between the nodes in a graph, 

\subsection{Hierarchical Distance Structural Encoding (HDSE)}\label{subsec:HDSE}

Firstly, we introduce a novel concept called \emph{graph hierarchy distance} (GHD), which is defined as follows.
\begin{definition} [\textbf{Graph Hierarchy Distance}] \label{def1}
    Given two nodes $u, v$ in graph $G$, and a graph hierarchy $(G^i,\phi_i)_{i\ge 0}$, %with $k\ge 1$, 
    the $k$-level  hierarchy distance between $u$ and $v$ is defined as
     \begin{align}\label{ghd}
\mathrm{GHD}^0\left( u,v \right) &= \mathrm{SPD}\left( u,v \right),
    \nonumber \\
    \mathrm{GHD}^k\left( u,v \right) &= \mathrm{SPD}\left( \phi_{k-1}...\phi_{0}(u), \phi_{k-1}...\phi_{0}(v) \right),
 \end{align}
    %VT that's not true: where $\phi_{k-1}(u), \phi_{k-1}(v) \in V^{k}$, and
    where $\mathrm{SPD} \left(\cdot , \cdot \right)$ is the shortest path distance between two nodes ($\infty$ if the nodes are not connected), and $\phi_{k-1}...\phi_{0}(\cdot)$ maps a node in $G^0$ to a node in $G^k$.   
   % Specifically, $\mathrm{GHD}^0\left( u,v \right) = \mathrm{SPD}\left( u,v \right)$.
\end{definition}

Note that the $k$-level hierarchy distance adheres to the definition of a metric, being zero for \(v = u\), invariably positive, symmetric, and fulfilling the triangle inequality. As illustrated on the left side of Figure \ref{fig:architecture}, it can be observed that $\mathrm{GHD}^0\left( v_1,v_{11} \right) = 7$, whereas $\mathrm{GHD}^1\left( v_1,v_{11} \right) = 2$.

Graph hierarchies have been observed to assist in identifying community structures in graphs that exhibit a clear property of tightly knit groups, such as social networks and citation networks \cite{girvan2002community}. They may also aid in prediction over graphs with a clear hierarchical structure, such as metal–organic frameworks or proteins. Fig.~\ref{fig:HDSE} shows that with the graph hierarchies generated by the Newman coarsening method, $\mathrm{GHD}^1$ is capable of capturing chemical motifs, including CF3 and aromatic rings.

\begin{figure}[t]   \center{\includegraphics[width=14cm]  {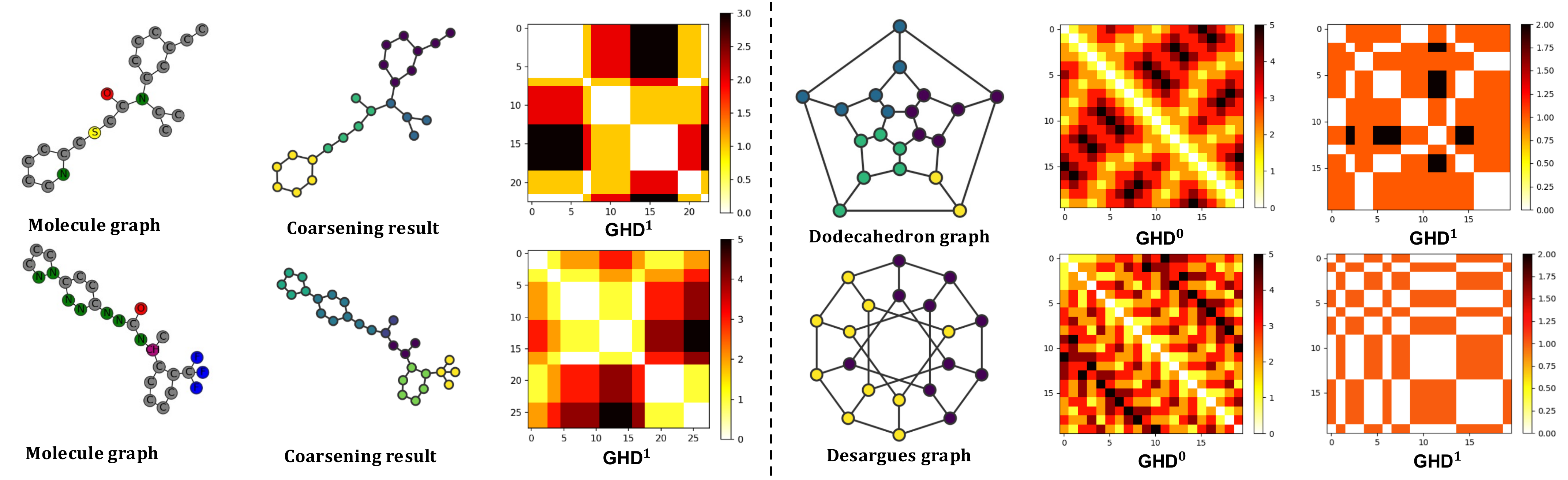}}   \caption{\label{1} Examples of graph coarsening results and hierarchy distances. Left: HDSE can capture chemical motifs such as CF3 and aromatic rings on molecule graphs. Right: HDSE can distinguish the Dodecahedron and Desargues graphs. The Dodecahedral graph has 1-level hierarchy distances of length 2 (indicated by the dark color), while the Desargues graph doesn’t. In contrast, the GD-WL test with SPD cannot distinguish these graphs \cite{zhang2023rethinking}.}\label{fig:HDSE} \vspace{-0.1 in} \end{figure}

% \subsection{Hierarchical Distance Structural Encoding}

%VT this is more introduction than proposed approach
%The motivation behind introducing distance as a bias in the Transformer stems from the need to encode the structural information of non-sequential data, into the model. While the Transformer architecture allows tokens to attend to information at any position, it requires explicit specification or encoding of positional dependency in the layers. For sequential data, absolute positional encoding \cite{vaswani2017attention} or relative positional encoding \cite{shaw2018self, dai2019transformer} can be used. However, for graphs, which have nodes arranged in a multi-dimensional spatial space and linked by edges, a different approach is required. Previous works, such as Graphormer \cite{ying2021transformers}, have attempted to incorporate the SPD as structural information into the attention weights to bias the attention module, but it fails to capture certain hierarchical structures in graphs, such as cycles or cliques. To address this limitation, w

% Based on GHD, we propose a novel spatial encoding method, where the spatial relation between nodes is measured using our hierarchy distance. % that represents the multi-level distance on the hierarchical structures. 
% Let $\mathrm{GHD}^k$ be the $k$-level hierarchy distance matrix and $\mathrm{GHD}^k_{i,j}$ is the $k$-level hierarchy distance between node $i$ and $j$. 

Based on GHD, we propose \emph{hierarchical distance structural encoding} (HDSE), defined for each pair of nodes $i, j \in V$:
\begin{equation}
\mathrm{D}_{i,j}=\left[ \mathrm{GHD}^0,\mathrm{GHD}^1,...,\mathrm{GHD}^{K} \right] _{i,j}\in \mathbb{R} ^{K+1},\label{eqn: HDSE D_ij}
\end{equation}
where $\mathrm{GHD}^k$ is the $k$-level hierarchy distance matrix, 
%and $\mathrm{GHD}^k_{i,j}$ is the $k$-level hierarchy distance between node $i$ and $j$,
and $K \in \mathbb{N}$ controls the maximum level of hierarchy considered.

We demonstrate the superior expressiveness of HDSE over SPD using recently proposed graph isomorphism tests inspired by the Weisfeiler-Leman algorithm \cite{weisfeiler1968reduction}. In particular, \cite{zhang2023rethinking} introduced the Generalized Distance Weisfeiler-Leman (GD-WL) Test 
%which incorporates distances by updating node colors. They 
and applied it to analyze a graph transformer architecture that employs $\mathrm{SPD}(i, j)$ as relative positional encodings. They proved that the graph transformer's maximum expressiveness is the GD-WL test with SPD. Here, we also use the GD-WL test to showcase the expressiveness of HDSE.
%Here, we show that by substituting  $\mathrm{SPD}(i, j)$ with HDSE $\mathrm{D}_{i,j}$, the GD-WL test using this new distance outperforms than the one using SPD.
%Remarkably, they demonstrate the architecture successfully solves edge biconnectivity problems.
%To demonstrate the superior expressiveness of HDSE over SPD, we can employ recently proposed graph isomorphism tests inspired by the Weisfeiler-Leman algorithm \cite{weisfeiler1968reduction}.

\begin{proposition}[\textbf{Expressiveness of HDSE}] \label{th1}
    GD-WL with HDSE $(\mathrm{D}_{i,j}) $ is strictly more expressive than GD-WL with the shortest path distance $\mathrm{SPD}(i, j)$.
\end{proposition}
The proof is provided in Appendix~\ref{th1-app}. Firstly, we show that the GD-WL test using HDSE can differentiate between any two graphs that can be distinguished by the GD-WL test with SPD. Next, we show that the GD-WL test with HDSE is capable of distinguishing the Dodecahedron and Desargues graphs (Figure \ref{fig:HDSE}) while the one with SPD cannot. 
%Specifically, the Dodecahedral graph has 1-level hierarchy distances of length 2 between pairs of nodes, while the Desargues graph doesn’t. In contrast, the GD-WL test with SPD cannot distinguish these graphs \cite{zhang2023rethinking}.

% \begin{figure}[htb]   \center{\includegraphics[width=8cm]  {HDSE_.jpg}}   \caption{\label{1} Coarsening results and hierarchy distance. (Up) HDSE is capable of identifying some chemical motifs such as CF3 and aromatic rings. (Down) GD-WL with HDSE can distinguish Dodecahedron and Desargues graphs, but GD-WL with SPD cannot.}   \vspace{-0.1 in} \label{fig:HDSE} \end{figure}

\subsection{Integrating HDSE in Graph Transformers} \label{integrate}\label{subsec:GT_HDSE}
% In each Transformer Layer, our objective is to leverage a structural-enhanced attention module to aggregate and transform node representations within the context of each node. Concretely, we expect that each node in a single Transformer layer can adaptively
%We integrate graph hierarchies into the transformer attention, on each layer, so that the information biases each node update accordingly.
%attend to all other nodes according to the graph hierarchy information. 
We integrate HDSE ($\mathrm{D}_{i,j}$) into the attention mechanism of each graph transformer layer to bias each node update. To achieve this, we use an end-to-end trainable function $\mathrm{Bias}:\mathbb{R} ^{K+1} \rightarrow \mathbb{R}$ to learn the biased structure weight $\mathrm{H}_{i,j} = \mathrm{Bias}(\mathrm{D}_{i,j})$. 
%while also incorporating other information through the attention layers. 
We limit the maximum distance length to a value~$L$, based on the hypothesis that detailed information loses significance beyond a certain distance. By imposing this limit, the model can extend acquired patterns to graphs of varying sizes not encountered in training. Specificially, we implement the function $\mathrm{Bias}$ using an MLP as follows:
%In particular, we restrict the maximum distance length to a value~$L$. Our hypothesis is that, beyond a certain distance, detailed information loses its significance. By imposing a limit on the maximum distance, the model becomes capable of extending the acquired patterns to graphs of varying sizes that may not have been encountered in the training phase. We implement the function $\mathrm{Bias}$ using an MLP as follows: 
% \veronika{Bias is nowhere in there? why do you need that?} \luo{I think this is a bias in attention, so we need a bias function. But you can modify it.}\veronika{you write The $\mathrm{Bias}$ is defined, but this symbol does not appear below}
\begin{align}
\mathrm{H}_{i,j}  &=\mathrm{MLP}\left( \left[ \mathbf{e}^0_{\mathrm{clip}^0_{i,j}},\cdots,\mathbf{e}^K_{\mathrm{clip}^K_{i,j}} \right] \right) \in \mathbb{R},
 \nonumber\\
\mathrm{clip}^k_{i,j}  &= \min \left( L, \mathrm{GHD}^k_{i,j} \right), 0\leq k \leq K,
\end{align}
where ${\left[ \mathbf{e}^k_0, \mathbf{e}^k_1, \cdots , \mathbf{e}^k_L \right]}_{0\leq k \leq K}\in \mathbb{R}^{d\times (L+1)}$ collects $L+1$ learnable feature embedding vectors $\mathbf{e}^k_i$ %\mathrm{Embed}^k_1, ... , \mathrm{Embed}^k_D$ 
for hierarchy level $k$. By embedding the hierarchy distances at different levels into learnable feature vectors, it may enhance the aggregation of multi-level graph information among nodes and expands the receptive field of nodes to a larger scale. We assume single-headed attention for simplified notation, but when extended to multiheaded attention, one bias is learned per distance per head. 
% In the fields of NLP, it is a common practice to bias Transformers towards sequential inputs by incorporating positional encodings into the token embeddings and there are already various variations of biased attention mechanisms \cite{shaw2018self,dai2019transformer,yang2019xlnet,tsai2019transformer}. 

We incorporate the learned biased structure weights $H$ to graph transformers, using the popular biased self-attention method proposed by \cite{ying2021transformers}, formulated as:
%To enable the graph transformer models to integrate hierarchy distance in a more general way, we introduce the widely-used form of the biased self-attention proposed by \cite{ying2021transformers}. With the modification to Eq. (1), the single-head biased attention can be formulated as:
 \begin{equation}
 	\label{eq:transformer2}\mathrm{Attention}\left(\mathbf{X}\right) =\mathrm{softmax} \left( \mathrm{A}+ \mathrm{H} \right) \mathbf{V},  \mathrm{A} = \frac{\mathbf{QK}^{\top}}{\sqrt{d'}},\\
 \end{equation}
% The calculation of the attention score $\mathrm{A}$ in traditional transformers involves taking the dot product between the encoded query and key embeddings. However, it overlooks the crucial structural information as it merely measures the similarity between the encoded features. 
%We incorporate the hierarchy distance bias $\mathrm{H}$ into the original attention score $\mathrm{A}$. 
%In this way, attention will not only focus on
%the node attributes but also on the multi-level topological structure when
% deciding the weights.
% It is worth mentioning that the hierarchy distance bias module serves as a 
where the original attention score $\mathrm{A}$ is directly augmented with $\mathrm{H}$. This approach is backbone-agnostic and can be seamlessly integrated into the self-attention mechanism of existing graph transformer architectures. Notably, we have the following results on expressiveness and generalization.
%Observe that we obtain a backbone-agnostic component, which can be integrated into the self-attention mechanism for existing graph transformer architectures. Importantly, we have the following expressivity and generalization results.

%VT you wrote this above
%Without loss of generality, we adopt the formulation in \cite{ying2021transformers} as the graph transformer architecture, and obtain 

% Given the output of the $l$-th transformer layer $\mathbf{H}^{l}$, we apply the regular transformer's computation:
% % \veronika{the below FFN... isn't specific to ying2021transformers, no? then I would put this into Sec 2} \luo{This is no specific, so I think it is appropriate to place it here, as other papers also put it here.}
% \begin{align*}
% \mathbf{\hat{H}}^{l-1}&= \mathrm{Attention}\left( \mathrm{LN}\left( \mathbf{H}^{l-1} \right) \right) +\mathbf{H}^{l-1}
% \\
% \mathbf{H}^{l}&= \mathrm{FFN}\left( \mathrm{LN}\left( \mathbf{\hat{H}}^{l-1} \right) \right) +\mathbf{\hat{H}}^{l-1}
% \end{align*}

% where $\mathrm{LN}$ denotes the layer-norm function and $\mathrm{FFN}$ is the feed-forward network.
\begin{proposition}\label{th2-1} The power of a graph transformer with HDSE to distinguish non-isomorphic graphs is at most equivalent to that of the GD-WL test with HDSE. With proper parameters and an adequate number of heads and layers, a graph transformer with HDSE can match the power of the GD-WL test with HDSE.
%A graph transformer with HDSE is at most as powerful as the GD-WL test with HDSE in distinguishing non-isomorphic graphs. Moreover, with proper parameters and using a sufficiently large number of heads and layers, the graph transformer with HDSE can be as powerful as the GD-WL test with HDSE.
\end{proposition}
See the proof in Appendix~\ref{th2-1-app}. This result provides a precise characterization of the expressivity power and limitations of graph transformers with HDSE. Combining Proposition \ref{th1} and \ref{th2-1} immediately yields the following corollary:

\begin{corollary}[\textbf{Expressiveness of Graph Transformers with HDSE}]\label{th2-2} There exists a graph transformer using HDSE (with fixed parameters), denoted as $\mathcal{M}$, such that $\mathcal{M}$ is more expressive than graph transformers with the same architecture using SPD, regardless of their parameters.
%There exists a graph transformer with HDSE, denoted as $\mathcal{M}$, such that $\mathcal{M}$ is more powerful than the graph transformer with SPD for any parameters.
%regardless of parameter settings.
\end{corollary}
This is a fine-grained expressiveness result of graph transformers with HDSE. It demonstrate the superior expressiveness of HDSE over
SPD in graph transformers.

\begin{proposition}[\textbf{Generalization of Graph Transformers with HDSE}](Informal)\label{th3} 
    For a semi-supervised binary node classification problem, suppose the label of each node $i\in V$ is determined by node features in the ``\textbf{hierarchical core neighborhood}'' $S_*^i=\{j: \mathrm{D}_{i,j}=D^*\}$ for a certain $D^*$, where $\mathrm{D}_{i,j}$ is HDSE defined in (\ref{eqn: HDSE D_ij}). Then, a properly initialized one-layer graph transformer equipped with HDSE can learn such graphs with a desired generalization error, while using SPD cannot guarantee satisfactory generalization.  %Then, for each node $i\in V$, a properly initialized one-layer graph transformer (i) without HDSE (ii) and only aggregate nodes from $S_*^i$ can achieve the same generalization error as learning with a one-layer graph transformer (a) with HDSE (b) aggregate all nodes in the graph without knowing $S_i^*$ in prior. 
\end{proposition}
The formal version and the proof are given in Appendix~\ref{th3-ap}. Proposition \ref{th3} is a corollary and extension of Theorem 4.1 of \cite{li2023what} from SPD to HDSE. It indicates that learning with HDSE can capture the labeling function characterized by the hierarchical core neighborhood, which is more general and comprehensive than the core neighborhood for SPD proposed in \cite{li2023what}. % graph learning based on the labeling function characterized by the hierarchical core neighborhood, which is more general and comprehensive than the core neighborhood for SPD proposed in \citep{li2023what}. %We first extend the notion of the core neighborhood from based on SPD, i.e., one-dimensional HDSE, to general high-dimensional HDSE. Then, by formulating the HDSE as a set of one-hot vectors, we can apply Theorem 4.3 of \citep{li2023what} to achieve the equivalence. 

% \textbf{Computational Complexity.} In contrast to the existing graph transformers, our proposed model entails additional 
% computational costs due to the inclusion of graph coarsening; see Sec.~\ref{sec:related}. 
% A comparison of the computational complexity of these algorithms is provided in Table~\ref{tab:tab1}. Note that most coarsening algorithms are highly efficient. Moreover, the computational cost of calculating the hierarchy distance between all pairs of nodes at $k$-level using Dijkstra's Algorithm \cite{dijkstra1959note} is $O((\alpha^k|V|)^3)$ where $\alpha$ is the coarsening ratio. Observe that, during training, we can consider these steps as pre-processing since they only need to be run once.

\subsection{Scaling HDSE to Large-scale Graphs}\label{subsec:HGT_HDSE}
% Since the pairwise shortest-path distance calculation might be too expensive for very large graphs with millions of nodes, we adapt our definition. Observe that we can still calculate meaningful distances for higher hierarchy levels, using

%\veronika{it's not clear how this is applied. eg for node classification. you may also need some assumption about the baseline transformer, eg that it samples nodes?} \luo{The meaning here is that in large graphs, we cannot perform SPD algorithms because they are too costy. However, we can calculate high-level hierarchy distance, which is efficient. It is independent of whether nodes are sampled or not; it is applicable both.}
%VT do we have datasets with milluions of nodes?
% also significant benefits - do we observe those?
%To train on large-scale graphs with millions of nodes, it is practically infeasible to compute the shortest path distance for all pairs of nodes. The computational and spatial costs associated with such an approach are prohibitively high. In such scenarios, employing a high-level hierarchy distance can be an efficient alternative, providing significant benefits with minimal additional computational cost. Therefore, we extend HDSE for large-scale graphs and the
For large graphs with millions of nodes, training and deploying a transformer with full global attention is impractical due to the quadratic cost. Some existing graph transformers address this issue by sampling nodes for attention computation~\cite{zhang2022hierarchical, zhu2023hierarchical}. While our HDSE can enhance these models, this sampling approach compromises the expressivity needed to capture interactions among all pairs of nodes.
%Intuitively, our HDSE offers valuable assistance in developing graph transformers which specifically sample nodes for attention in large graphs \cite{zhang2022hierarchical, zhu2023hierarchical}. However, such strategies compromise the expressivity required to capture interactions among all pairs of nodes. 
However, in the NLP domain, Linformer \cite{wang2020linformer} utilizes a learnable low-rank projection matrix $\mathbf{\tilde{P}}$ to reduce the complexity of the self-attention module to linear levels:
 \begin{equation}\mathrm{Attention}\left( \mathbf{X} \right)=
  % \sum_{k=S}^{K}
  {\mathrm{softmax} \left( {\mathbf{X}\mathrm{W}_\mathbf{Q}(\mathbf{\tilde{P}X}\mathrm{W}_\mathbf{K})^{\top}}/{\sqrt{d'}}\right) \mathbf{\tilde{P}X}\mathrm{W}_\mathbf{V}}.\label{eq:linformer}\end{equation}
Inspired by Linformer, models like GOAT \cite{pmlr-v202-kong23a} and Gapformer \cite{liu2023gapformer} in the graph domain also employ projection matrices to reduce the number of nodes by projecting them onto fewer super-nodes, consequently compressing the input node feature matrix $\mathbf{X}$ into a smaller dimension.
This transformation enables the aggregation of information from super-nodes and reduces the quadratic complexity of attention computation to linear complexity.
Here, we can replace the projection matrix with our coarsening partition matrix $P$ (see Sec.~\ref{sec:hie}) to obtain representations of coarser graphs at higher levels.
Observe that we can still calculate meaningful distances at these higher hierarchy levels, using a \emph{high-level} HDSE as follow:
%To address this, we can use the projection matrix $P$ to transform the input node feature $\mathbf{X}$ into a low-dimensional cluster space. This transformation allows input nodes to aggregate information from these compact clusters, thereby reducing the quadratic complexity of attention computation to linear complexity. Observe that we can still calculate meaningful distances for this linear attention, usinga \emph{high-level} HDSE as follows:
% \veronika{I think we need to mark the H somehow since it's novel symbol/def, eg H'}
% $$
% \mathrm{D'}_{i,j}=\left[ \mathrm{GHD}^S,...,\mathrm{GHD}^{K} \right] _{i,j}\in \mathbb{R} ^{K+1-S},
% $$

\begin{equation}
\mathrm{D}^c_{i,j}= \left[ \mathrm{GHD}^c(\prod_{l=0}^{c-1} P^l),...,\mathrm{GHD}^{K}(\prod_{l=0}^{c-1} P^l) \right]_{i,j}{}_{1 \le c \le K}, 
\end{equation}
where each row of the projection matrix $P^l$ (see Sec.~\ref{sec:hie}) is a one-hot vector representing the $l$-level cluster that an input node belongs to, and $\mathrm{D}^c \in \mathbb{R} ^{|V^{0}|\times |V^{c}|\times (K+1-c)}$. Note that $\mathrm{GHD}^m(\prod_{l=0}^{c-1} P^l)$, $c \le m \le K$ computes distances from input nodes to clusters at $c$-level graph hierarchy. In practice, these distances can be directly obtained by calculating the hierarchy distance between all node pairs at the $c$-level. When \( c = 0 \), \( \mathrm{D}^c \) becomes \( \mathrm{D} \) in Eq \ref{eqn: HDSE D_ij}. In this way, the high-level HDSE establishes attention between nodes in the input graph $G$ and clusters at high level hierarchies. For example, we can integrate the high-level HDSE into Linformer by adapting Equation (\ref{eq:linformer}):
% \vspace{-0.2 in}
\begin{align}
 	\label{eq:transformer3}\mathrm{Attention}\left( \mathbf{X} \right)&=
  % \sum_{k=S}^{K}
  {\mathrm{softmax} \left( \frac{\mathbf{X}\mathrm{W}_\mathbf{Q}(\mathbf{X}^k\mathrm{W}_\mathbf{K})^{\top}}{\sqrt{d'}} + \mathrm{H}^k \right) \mathbf{X}^k\mathrm{W}_\mathbf{V}},  \nonumber \\
  % \mathbf{X}^{k}&=\mathrm{diag}^{-1}(\mathbf{1}P^{k-1}){P^{k-1}}^{\top}\mathbf{X}^{k-1}\\
  \mathrm{H}^k&=\mathrm{Bias}(\mathrm{D}^k) \in \mathbb{R} ^{|V^{0}|\times |V^{k}|}, 
% \mathrm{D}^k&=\bigoplus_{m=S}^K{\left[\mathrm{GHD}^m(\prod_{i=0}^{k-1} P^i)^{\top} \right]} \\
  % \mathrm{D}^k&= \left[ \mathrm{GHD}^k(\prod_{i=0}^{k-1} P^i),...,\mathrm{GHD}^{K}(\prod_{i=0}^{k-1} P^i) \right]
\end{align}
where 
% parameters $S \ge 1$ is the minimum level of hierarchy, 
% $\bigoplus$ denotes the concatenation, 
 $\mathbf{X}^{k} \in \mathbb{R} ^{|V^{k}|\times d}$ (see Sec.~\ref{sec:hie}) represents the features of clusters at $k$-level, and $\mathrm{Bias}: \mathbb{R}^{K+1-k} \mapsto \mathbb{R}$ is a end-to-end trainable function as defined in Sec.~\ref{integrate}.

\vspace{-0.05 in}
\section{Evaluation}
We evaluate our proposed HDSE on 18 benchmark datasets, and show state-of-the-art performance in many cases. Primarily, the following questions are investigated:
% \vspace{-0.05 in}
\begin{itemize}[leftmargin=*,noitemsep,topsep=0pt] %[leftmargin=*] % ,noitemsep,topsep=0pt
	\item Can \textbf{HDSE improve upon existing graph transformers}, and how does the choice of \textbf{coarsening algorithm} affect performance? (\textbf{Sec. \ref{graph-level}})
	\item Does our \textbf{adaptation for large graphs} also \textbf{show effectiveness}, is it marked by \textbf{efficiency}, and how does \textbf{high-level HDSE} impact the performance? (\textbf{Sec. \ref{node-level}})
\end{itemize}

\subsection{Experiment Setting}
\label{ssec:Datasets}
%  That was a bit too much duplication!
\textbf{Datasets.} We consider 
%To provide comprehensive evaluation, we conducted experiments on five benchmarks taken from the Benchmarking GNNs work \cite{dwivedi2023benchmarking}, as well as two benchmarks from the newly developed Long-Range Graph Benchmark \cite{dwivedi2022long}. These benchmarks encompass 
various graph learning tasks from popular benchmarks as detailed below and in Appendix~\ref{ap-b}. 
% \vspace{-0.1 in}
% , such as node classification, graph classification, and graph regression. %They specifically highlight the importance of graph structure encoding, node clustering, and the ability to learn long-range dependencies. Additionally, to assess the generalizability of HDSE to large-scale graphs, we conduct experiments on larger datasets, including citation graphs and web graphs. 
% Further details can be found in the appendix. % concerning the datasets can be found in Appendix.
\begin{itemize}[leftmargin=*,noitemsep,topsep=0pt]
	\item \textbf{Graph-level Tasks}\textbf{.} 
    For graph classification and regression, we evaluate our method on five datasets from Benchmarking GNNs \cite{dwivedi2023benchmarking}: ZINC, MNIST, CIFAR10, PATTERN, and CLUSTER. We also test on two peptide graph benchmarks from the Long-Range Graph Benchmark \cite{dwivedi2022long}: Peptides-func and Peptides-struct, focusing on classifying graphs into 10 functional classes and regressing 11 structural properties, respectively. We follow all evaluation protocols suggested by \cite{rampavsek2022recipe}.
    % They are both composed of atomic graphs of peptides retrieved from SATPdb. In Peptides-func the prediction is graph classification into 10 functional classes. For Peptides-struct the task is graph regression of 11 structural properties. %We follow all suggested experimental settings.
    % \vspace{-0.05 in}
	\item\textbf{Node Classification on Large-scale Graphs.} %\cite{kipf2017semisupervised, chien2020adaptive, pei2019geom}\textbf{.} 
 We consider node classification over the citation graphs Cora, CiteSeer and PubMed \cite{kipf2017semisupervised}, the Actor co-occurrence graph \cite{chien2020adaptive}, and the Squirrel and Chameleon page-page networks from \cite{rozemberczki2021multi}, all of which have 1K-20K nodes. Further, we extend our evaluation to larger datasets from the Open Graph Benchmark (OGB) \cite{hu2020open}: ogbn-arxiv, arxiv-year, ogbn-papers100M, ogbn-proteins and ogbn-products, with node numbers ranging from 0.16M to 0.1B. We maintain all the experimental settings as described in \cite{wu2023simplifying}. 
\end{itemize}
% \vspace{-0.1 in}
% \veronika{this listing that was below w/o explanations doesn't provide anything but consumes space in a "filling" way - which makes it seem we have not enough contributions to fill space differently. move the references into the table!} \luo{tbd: the table cannot fit in, so we really need to consider putting it here.}

\textbf{Baselines.} 
We compare our method to the following prevalent GNNs: GCN \cite{kipf2017semisupervised}, GIN \cite{xu2018powerful}, GAT \cite{velivckovic2018graph}, GatedGCN \cite{bresson2017residual}, GatedGCN-RWSE \cite{dwivedi2021graph}, JKNet \cite{xu2018representation}, APPNP \cite{gasteiger2018predict}, SGC \cite{wu2019simplifying}, PNA \cite{corso2020principal}, GPRGNN \cite{chien2020adaptive}, SIGN \cite{rossi2020sign}, H2GCN \cite{zhu2020beyond}; and other recent GNNs with SOTA performance:
% DGN \cite{beaini2021directional}, 
% GSN \cite{bouritsas2022improving}, 
LINKX \cite{lim2021large},
CIN \cite{bodnar2021weisfeiler}, 
% CRaW1 \cite{toenshoff2021graph}, 
GIN-AK+ \cite{zhao2021stars}, HC-GNN \cite{zhong2023hierarchical}. In terms of transformer models, we consider \textbf{GT}\cite{dwivedi2020generalization}, Graphormer \cite{ying2021transformers}, SAN \cite{kreuzer2021rethinking}, Coarformer \cite{kuang2021coarformer}, ANS-GT \cite{zhang2022hierarchical}, EGT \cite{hussain2022global}, NodeFormer \cite{wu2022nodeformer}, Specformer \cite{bo2023specformer}, MGT \cite{ngo2023multiresolution}, AGT \cite{ma2023rethinking}, HSGT \cite{zhu2023hierarchical}, Graphormer-GD \cite{zhang2023rethinking}, \textbf{SAT} \cite{chen2022structure}, \textbf{GOAT} \cite{pmlr-v202-kong23a}, Gapformer \cite{liu2023gapformer}, Graph ViT/MLP-Mixer \cite{he2023generalization}, LargeGT \cite{dwivedi2023graph}, NAGphormer \cite{chen2022nagphormer}, Exphormer \cite{shirzad2023exphormer}, DRew \cite{gutteridge2023drew}, VCR-GT \cite{fu2024vcrgraphormer}, CoBFormer \cite{xing2024less} and recent SOTA graph transformers \textbf{GraphGPS} \cite{rampavsek2022recipe}, \textbf{GRIT} \cite{ma2023graph}, SGFormer \cite{wu2023simplifying}.

% Please note that several of the works on transformers using hierarchical information have not published their code yet, consider very different data, or employ settings which make direct comparison impossible. 
% Please see the appendix for more details as well as for all baseline references.
%\veronika{that's a strong clame and needs to ber verified. can you do this at least in terms of papers with code or so}

% \begin{table}
% \vspace{-0.5 in}
% 	\centering
% 	\begin{tabular}{lccc}
% 		\toprule
% 		  & $K=0$ & $K=1$ & $K=2$\\
%             \midrule %
%             P-func $\uparrow$ & 0.687 $\pm$ \footnotesize{0.013}  & 0.704 $\pm$ \footnotesize{0.004} &  0.701 $\pm$ \footnotesize{0.006} \\
%             ZINC  $\downarrow$ & 0.069 $\pm$ \footnotesize{0.003}  & 0.063 $\pm$ \footnotesize{0.003} &  0.064 $\pm$ \footnotesize{0.004} \\
% 		\bottomrule
% 	\end{tabular}

% 	\caption{Sensitivity analysis on the maximum hierarchy level $K$ of GraphGPS + HDSE on Peptides-func and ZINC.}
%      \vspace{-0.2 in}
% 	\label{tab:tab6}
% \end{table}

\textbf{Models + HDSE.} 
We integrate HDSE into GT, SAT, GraphGPS, GRIT (and ANS-GT in appendix) \emph{only modifying their self-attention module} by Eq. \ref{eq:transformer2}.
For fair comparisons, we use the same hyperparameters (including the number of layers, batch size, hidden dimension etc.), PE and readout as the baseline transformers.
Given one of the baseline transformers \textbf{M}, we denote the modified model using HDSE by \textbf{M + HDSE}. Additionally, we integrate our high-level HDSE into \textbf{GOAT}, denoted as \textbf{GOAT + HDSE}. 
We fix the maximum distance length $L=30$ and vary the maximum hierarchy level $K$ in $\{0,1,2\}$ in all experiments. A sensitivity analysis of these two parameters is presented in Appendix~\ref{ap-c}.
% If not specified otherwise, we chose the maximum hierarchy level $K=1$ and maximum distance length $L=30$ in all experiments. 
% \textbf{Computational Runtime.} In contrast to the existing graph transformers, our proposed model entails additional 
% computational costs due to the inclusion of graph coarsening; see Sec.~\ref{sec:related}. 
% A comparison of the computational complexity of these algorithms is provided in Table~\ref{tab:tab1}. Note that most coarsening algorithms are highly efficient. Moreover, the computational cost of calculating the hierarchy distance between all pairs of nodes at $k$-level using Dijkstra's Algorithm \cite{dijkstra1959note} is $O((\alpha^k|V|)^3)$ where $\alpha$ is the coarsening ratio. Observe that, during training, we can consider these steps as pre-processing since they only need to be run once. 
During training, the steps of coarsening and distance calculation \cite{dijkstra1959note} can be treated as pre-processing, since they only need to be run once. We detail the choice and runtime of coarsening algorithms for HDSE in the appendix. 
% For large-scale graphs, we use the efficient METIS algorithm, even partitioning the Amazon2M graph in under five minutes. 
Detailed experimental setup and hyperparameters are in Appendix~\ref{ap-b} due to space constraints.

% and please also refer to the Appendix for the results on integrating HDSE into ANS-GT for large-scale graphs.  
% upper-case A rather if you have a number or so, ie if you refer to a specific one

\begin{table*}[t]
\scriptsize
	\centering
        \vspace{-0.1 in}
         \caption{Test performance in five benchmarks from \cite{dwivedi2023benchmarking}. Shown is the mean $\pm$ s.d. of 5 runs with different random seeds. Baseline results were obtained from their respective original papers. $^*$ indicates a statistically significant difference against the baseline w/o HDSE from the one-tailed t-test. Highlighted are the top \textbf{\textcolor{darkgreen}{first}}, \textbf{\textcolor{orange}{second}} and \textbf{\textcolor{deepred}{third}} results.}
        % \vspace{-0.1 in}
	\begin{tabular}{lccccc}
		\toprule
		{Model} & {ZINC} & {MNIST} & {CIFAR10} & {PATTERN} & {CLUSTER}\\
		% \cmidrule{2-3,4-5,6-7}
		& MAE $\downarrow$ & Accuracy $\uparrow$ & Accuracy $\uparrow$ & Accuracy $\uparrow$ & Accuracy $\uparrow$\\
             % \midrule %
         % Average Clust. Coeff. & 0.006 & 0.477 & 0.454 & 0.427 & 0.316\\
		\midrule %
        GCN & 0.367 $\pm$ 0.011 & 90.705 $\pm$ 0.218 & 55.710 $\pm$ 0.381 & 71.892 $\pm$ 0.334 & 68.498 $\pm$ 0.976  \\
        GIN & 0.526 $\pm$ 0.051 & 96.485 $\pm$ 0.252 & 55.255 $\pm$ 1.527 & 85.387 $\pm$ 0.136 & 64.716 $\pm$ 1.553  \\ 
        % GAT & 0.384 $\pm$ 0.007 & 95.535 $\pm$ 0.205 & 64.223 $\pm$ 0.455 & 78.271 $\pm$ 0.186 & 70.587 $\pm$ 0.447  \\
        GatedGCN & 0.282 $\pm$ 0.015 & 97.340 $\pm$ 0.143 & 67.312 $\pm$ 0.311 & 85.568 $\pm$ 0.088 & 73.840 $\pm$ 0.326  \\
        % GatedGCN-LSPE & 0.090 $\pm$ 0.001 &  – &  – &  – &  –  \\
        PNA & 0.188 $\pm$ 0.004 & 97.940 $\pm$ 0.120 & 70.350 $\pm$ 0.630 &  – &  –  \\
        % DGN & 0.168 $\pm$ 0.003 &  –  & 72.838 $\pm$ 0.417 & 86.680 $\pm$ 0.034 &  –  \\
        % GSN & 0.101 $\pm$ 0.010 &  – &  – &  – &  –  \\
        CIN & 0.079 $\pm$ 0.006 &  – &  – &  – &  –  \\
        % CRaWl & 0.085 $\pm$ 0.004 & 97.944 $\pm$ 0.050 & 69.013 $\pm$ 0.259 &  – &  –  \\
        GIN-AK+ & 0.080 $\pm$ 0.001 &  –  & 72.190 $\pm$ 0.130 & 86.850 $\pm$ 0.057 &  –  \\
        \midrule %
        SGFormer & 0.306 $\pm$ 0.023 & – & – & 85.287 $\pm$ 0.097 & 69.972 $\pm$ 0.634   \\
        \midrule %
        SAN & 0.139 $\pm$ 0.006 &  – &  –  & 86.581 $\pm$ 0.037 & 76.691 $\pm$ 0.650  \\
        % Graphormer & 0.122 $\pm$ 0.006 &  – &  – &  – &  –  \\
        Graphormer-GD & 0.081 $\pm$ 0.009 & – & – & – & –  \\
        Specformer & \textbf{\textcolor{deepred}{0.066 $\pm$ 0.003}} & – & – & – & –  \\
        EGT & 0.108 $\pm$ 0.009 & 98.173 $\pm$ 0.087  & 68.702 $\pm$ 0.409 & 86.821 $\pm$ 0.020 & \textbf{\textcolor{deepred}{79.232 $\pm$ 0.348}}  \\
        Graph ViT/MLP-Mixer &  0.073 $\pm$ 0.001  &  97.422 $\pm$ 0.110 &  73.961 $\pm$ 0.330 &  – &  –  \\
        Exphormer &  -  &  \textbf{\textcolor{darkgreen}{98.550 $\pm$ 0.039}} &  74.696 $\pm$ 0.125 &  86.742 $\pm$ 0.015 &  78.071 $\pm$ 0.037 \\
        \midrule %
        GT & 0.226 $\pm$ 0.014 &  90.831 $\pm$ 0.161  &  59.753 $\pm$ 0.293 & 84.808 $\pm$ 0.068  &  73.169 $\pm$ 0.622 \\
        \textbf{GT + HDSE}  & 0.159 $\pm$ 0.006$^*$ &  94.394 $\pm$ 0.177$^*$ &  64.651 $\pm$ 0.591$^*$ & 86.713 $\pm$ 0.049$^*$ &  74.223 $\pm$ 0.573$^*$   \\
        \midrule %
        SAT & 0.094 $\pm$ 0.008 &  – &  – & 86.848 $\pm$ 0.037 & 77.856 $\pm$ 0.104  \\
        \textbf{SAT + HDSE}  & 0.084 $\pm$ 0.003$^*$ &  – &  – & \textbf{\textcolor{deepred}{86.933 $\pm$ 0.039$^*$}} & 78.513 $\pm$ 0.097$^*$ \\
        \midrule %
        GraphGPS & 0.070 $\pm$ 0.004 & 98.051 $\pm$ 0.126 & 72.298 $\pm$ 0.356 & 86.685 $\pm$ 0.059 & 78.016 $\pm$ 0.180 \\
        \textbf{GraphGPS + HDSE} & \textbf{\textcolor{orange}{0.062 $\pm$ 0.003$^*$}} & \textbf{\textcolor{deepred}{98.367 $\pm$ 0.106$^*$ }}& \textbf{\textcolor{deepred}{76.180 $\pm$ 0.277$^*$}} & 86.737 $\pm$ 0.055 & 78.498 $\pm$ 0.121$^*$ \\
        \midrule %
        GRIT & \textbf{\textcolor{darkgreen}{0.059 $\pm$ 0.002}} & 98.108 $\pm$ 0.111  & \textbf{\textcolor{orange}{76.468 $\pm$ 0.881}} & \textbf{\textcolor{orange}{87.196 $\pm$ 0.076}} & \textbf{\textcolor{darkgreen}{80.026    $\pm$ 0.277}}  \\
        \textbf{GRIT + HDSE} & \textbf{\textcolor{darkgreen}{0.059 $\pm$ 0.004}} & \textbf{\textcolor{orange}{98.424 $\pm$ 0.124$^*$}}  & \textbf{\textcolor{darkgreen}{76.473 $\pm$ 0.429}} & \textbf{\textcolor{darkgreen}{87.281 $\pm$ 0.064}} & \textbf{\textcolor{orange}{79.965    $\pm$ 0.203}}  \\
        \bottomrule
	\end{tabular}
    \vspace{-0.1 in}
	\label{tab:tab2}
\end{table*}

\vspace{-0.1 in}
\begin{table}[h]
 % \vspace{-0.2 in}
 \begin{minipage}[t]{0.52\linewidth}
 \caption{Test performance on two peptide datasets from Long-Range Graph Benchmarks (LRGB) \cite{dwivedi2022long}.}
	\centering
        \scriptsize
	\begin{tabular}{lcc}
		\toprule
		Model & Peptides-func & Peptides-struct\\
		&AP $\uparrow$ & MAE $\downarrow$\\
		\midrule %
		GCN  & 0.5930 \scriptsize{$\pm$ 0.0023}  & 0.3496 \scriptsize{$\pm$ 0.0013} \\
            GINE  & 0.5498 \scriptsize{$\pm$ 0.0079}  & 0.3547 \scriptsize{$\pm$ 0.0045} \\
            GatedGCN  & 0.5864 \scriptsize{$\pm$ 0.0035}  & 0.3420 \scriptsize{$\pm$ 0.0013} \\
            GatedGCN+RWSE  & 0.6069 \scriptsize{$\pm$ 0.0035}  & 0.3357 \scriptsize{$\pm$ 0.0006} \\
            \midrule %
            GT  & 0.6326 \scriptsize{$\pm$ 0.0126}  & 0.2529 \scriptsize{$\pm$ 0.0016} \\
            % SAN+LapPE  & 0.6384 \scriptsize{$\pm$ 0.0121} &  0.2683 \scriptsize{$\pm$ 0.0043} \\
            SAN+RWSE  & 0.6439 \scriptsize{$\pm$ 0.0075} &  0.2545 \scriptsize{$\pm$ 0.0012} \\
            MGT+WavePE & 0.6817 \scriptsize{$\pm$ 0.0064}  &  \textbf{\textcolor{darkgreen}{0.2453 \scriptsize{$\pm$ 0.0025}}} \\
            GRIT  & \textbf{\textcolor{deepred}{0.6988 \scriptsize{$\pm$ 0.0082}}} &  \textbf{\textcolor{deepred}{0.2460 \scriptsize{$\pm$ 0.0012}}} \\
            Exphormer & 0.6527 \scriptsize{$\pm$ 0.0043}	& 0.2481 \scriptsize{$\pm$ 0.0007} \\
            Graph ViT/MLP-Mixer &  0.6970 \scriptsize{$\pm$ 0.0080} & 0.2475 \scriptsize{$\pm$ 0.0015} \\
            DRew & \textbf{\textcolor{orange}{0.7150 \scriptsize{$\pm$ 0.0044}}}	& 0.2536 \scriptsize{$\pm$ 0.0015} \\
            \midrule 
            GraphGPS  & 0.6535 \scriptsize{$\pm$ 0.0041} & 0.2500 \scriptsize{$\pm$ 0.0012} \\
            \textbf{GraphGPS + HDSE} & \textbf{\textcolor{darkgreen}{0.7156 \scriptsize{$\pm$ 0.0058}$^*$}} & \textbf{\textcolor{orange}{0.2457 \scriptsize{$\pm$ 0.0013}$^*$}} \\
		\bottomrule
	\end{tabular}
  % \vspace{-0.05 in}
 % \vspace{-0.1 in}
	\label{tab:tab3}
 \end{minipage}
    \hfill
    \begin{minipage}[t]{0.45\linewidth}
	\centering
 \scriptsize
 % \vspace{-0.2 in}
 \caption{Ablations on design choices of coarsening algorithms on ZINC.}
	\begin{tabular}{lcc}
		\toprule
		Model & Coarsening algorithm & ZINC \\
            % \cmidrule(lr){3-4}
              &   & MAE $\downarrow$ \\
            \midrule %
            \multirow{6}{*}{SAT}
            & w/o & 0.094 $\pm$ 0.008   \\
            % & w/ SPD & 0.093 $\pm$ 0.007   \\
            \cmidrule(lr){2-3}
            & METIS  & 0.089 $\pm$ 0.005   \\
		& Spectral & 0.088 $\pm$ 0.004  \\
		& Loukas & \textbf{\textcolor{darkgreen}{0.084 $\pm$ 0.003}}   \\
            & Newman  & \textbf{\textcolor{orange}{0.087 $\pm$ 0.002}}   \\
		& Louvain  & 0.088 $\pm$ 0.003   \\
		\midrule %
            \multirow{6}{*}{GraphGPS}
            & w/o  & 0.070 $\pm$ 0.004   \\
            % & w/ SPD  & 0.069 $\pm$ 0.003 &  0.687 $\pm$ 0.013  \\
            \cmidrule(lr){2-3}
            & METIS  & 0.069 $\pm$ 0.002   \\
		& Spectral  & \textbf{\textcolor{orange}{0.063 $\pm$ 0.003}}  \\
		& Loukas & 0.067 $\pm$ 0.002   \\
            & Newman  & \textbf{\textcolor{darkgreen}{0.062 $\pm$ 0.003}}   \\
		& Louvain  & 0.064 $\pm$ 0.002 \\
		\bottomrule
	\end{tabular}
  % \vspace{-0.05 in}
	\label{tab:tab5}
 \end{minipage}
 % \vspace{-0.15 in}
\end{table}
\vspace{-0.1 in}
\subsection{Results on Graph-level Tasks}\label{graph-level}

\textbf{Benchmarks from Benchmarking GNNs, Table \ref{tab:tab2}.} %As shown in  we first benchmark our method and other competitive models on five datasets: ZINC, MNIST, CIFAR10, PATTERN, and CLUSTER. 
We observe that nearly all four baseline graph transformers, when combined with HDSE, demonstrate performance improvements.
% We also checked if there is a correlation with the cluster structure according to \cite{hu2020open}, by computing clustering coefficients, but we do not observe a direct correlation. Notably, the ZINC dataset, which comprises small molecules, has a low clustering coefficient; however, our HDSE shows a significant improvement on it. This improvement could be attributed to the HDSE capturing chemical motifs that cannot be captured by the clustering coefficient, as illustrated in Figure \ref{fig:HDSE}. 
Note that the enhancement is overall especially considerable for GT.
On CIFAR10, we also obtain similar improvement for GraphGPS. 
Among them, GT shows the greatest enhancement and becomes competitive to more complex models. Our model attains the best or second-best mean performance for all datasets. While the improvement for GRIT is smaller, as its relative random walk probabilities (RRWP) already incorporate distance information \cite{ma2023graph}, we still observe improvements in three datasets. This indicates that HDSE can provide additional information beyond what is captured by RRWP. Notably, it is observed that the SOTA SGFormer tailored for large-scale node classification underperforms in graph-level tasks.
% These results highlight the superior performance of \textbf{Models + HDSE} over various methods on small to medium-sized datasets. This indicates the significant impact of multi-level graph structures, implicitly represented by the hierarchy distance on the performance of graph transformers. 

\textbf{Long-Range Graph Benchmark, Table \ref{tab:tab3}.} We consider GraphGPS due to its superior performance.
%+ HDSE on two LRGB datasets, which aim to test a method's ability to capture long-range dependencies in input graphs. O
Note that our HDSE module only introduces a small number of additional parameters, allowing it to remain within the benchmark's $\sim$500k model parameter budget. %On both datasets, our method achieves the highest mean performance. 
% you could have selected these two, since it's not a large number I wouldn't stress the two
In the Peptides-func dataset, HDSE yields a significant improvement of 6.21\%. This is a promising result and hints at potentially great benefits for macromolecular data more generally.

\textbf{Impact of Coarsening Algorithms, Table \ref{tab:tab5}, Figure \ref{fig:attention}.} % It is worth noting that different graph coarsening algorithms result in distinct multi-level graph structures. Therefore, to investigate the impact of these diverse coarsening algorithms on our architectural design, w
We conduct several ablation experiments on ZINC and observe that the dependency on the coarsening varies with the transformer backbone. For instance, the multi-level graph structures extracted by the Newman algorithm yields the largest improvement for GraphGPS. 
More generally, our experiments indicate that, Newman works best for molecular graphs. 
% However, its complexity is too high, and thus making it unsuitable for large-scale graphs. In the case of large-scale graphs, METIS provides a good trade-off between efficiency and performance.
% DOn't all do that?:, possibly due to its ability to identify crucial community structures. 
We visualize the attention scores on the ZINC and Peptides-func datasets respectively, as shown in Figure \ref{fig:attention}. The results indicate that our HDSE method successfully leverages hierarchical structure. %, which may serve as a key factor in the effectiveness of HDSE.

% \textbf{Synthetic Community Graphs, Table \ref{tab:syn}.} These synthetic community graphs, generated using the Erdos-Renyi model \cite{erdds1959random}, have adjacency matrices that obey the certain clustered structure. As evidenced in Table \ref{tab:syn}, the Graph Transformer struggles to detect such structures; and solely utilizing SPD proves inadequate; however, our HDSE, enriched with coarsening structural information, effectively captures these structures.
\begin{table*}[t]
    \centering
    \vspace{-0.1 in}
    {\tiny
    \caption{\textbf{Node classification} on large-scale graphs (\%).
     The baseline results were primarily taken from \citep{wu2023simplifying}, with the remaining obtained from their respective original papers. OOM indicates out-of-memory when training on a GPU with 24GB memory.}
    % \vspace{-0.1 in}
    \begin{tabular}{c|ccc|ccccc}
        \toprule
        Model    & Actor      & Squirrel   & Chameleon  & ogbn-proteins & ogbn-arxiv   & arxiv-year    & ogbn-products  &ogbn-100M      \\
         \midrule 
        \# nodes    & 7,600      & 2223       & 890        & 132,534       & 169,343      & 169,343      & 
        2,449,029 & 111,059,956    \\
        \# edges & 29,926 & 46,998 & 8,854 & 39,561,252 & 1,166,243 & 1,166,243 & 61,859,140 &1,615,685,872 \\
         & Accuracy$\uparrow$ &Accuracy$\uparrow$ &Accuracy$\uparrow$ &ROC-AUC$\uparrow$ &Accuracy$\uparrow$ &Accuracy$\uparrow$ &Accuracy$\uparrow$ &Accuracy$\uparrow$ \\
        \midrule %
        GCN        & 30.1 ± 0.2 & 38.6 ± 1.8 & 41.3 ± 3.0 & 72.51 ± 0.35  & 71.74 ± 0.29 & 46.02 ± 0.26 & 75.64 ± 0.21  & 62.04 ± 0.27\\
        GAT        & 29.8 ± 0.6 & 35.6 ± 2.1 & 39.2 ± 3.1 & 72.02 ± 0.44           & 71.95 ± 0.36  & 50.27 ± 0.20 & 79.45 ± 0.59             & 63.47 ± 0.39          \\
        SGC        & 27.0 ± 0.9 & 39.3 ± 2.3 & 39.0 ± 3.3 & 70.31 ± 0.23          & 67.79 ± 0.27 & -  & 76.31 ± 0.37 & 63.29 ± 0.19 \\
        % JKNet &30.8 ± 0.7 &39.4 ± 1.6 &39.4 ± 3.8 & -          & 72.19 ± 0.21          & -          & -           &- \\
        % APPNP &31.3 ± 1.5 &35.3 ± 1.9 &38.4 ± 3.5 & -          & -          & -          & -           &-\\
        % H2GCN &34.4 ± 1.7 &35.1 ± 1.2 &38.1 ± 4.0 & -          & -          & -                    & -\\
        SIGN &36.5 ± 1.0 &40.7 ± 2.5 &41.7 ± 2.2 &71.24 ± 0.46 &71.95 ± 0.11 &- &80.52 ± 0.16 & \textbf{\textcolor{deepred}{65.11 ± 0.14}}\\
        LINKX &- &- &- &- &66.18 ± 0.33 & 53.53 ± 0.36 &71.59 ± 0.71 & -\\
        HC-GNN & -          & -          & -          & -             & \textbf{\textcolor{deepred}{72.79 ± 0.25}} & 
        - & -             & - \\
        \midrule 
        Graphormer &OOM &40.9 ± 2.5 &41.9 ± 2.8 &OOM &OOM &OOM &OOM &OOM\\
        SAT & - & - & - & OOM & OOM & OOM & OOM &OOM\\
        \midrule 
        ANS-GT &- &- &- &74.67 ±  0.65 &72.34 ±  0.50 &- &80.64 ± 0.29 &-\\
        AGT &- &- &- &- &72.28 ± 0.38 &47.38 ± 0.78 &- &-\\
        HSGT &- &- &- &78.13 ±  0.25 &72.58 ±  0.31 &- &\textbf{\textcolor{deepred}{81.15 ± 0.13}} &- \\
        \midrule 
        GraphGPS &33.1 ± 0.8 & - &36.2 ± 0.6 & - & 70.97 ± 0.41 & - & OOM &OOM\\
        Gapformer    & -          & -          & -                     & - & 71.90 ± 0.19 & 
        - & -            & -           \\
        LargeGT & -          & -          & -          & -             & - & 
        - & -            & 64.73 ± 0.05          \\
        VCR-GT & -          & -          & -          & -             & - & 
        \textbf{\textcolor{orange}{54.15 ± 0.09}} & -            & -          \\
        NAGphormer& -          & -   & -   & -     & 70.13 ± 0.55          & -              &  73.55 ± 0.21            & -          \\
        Exphormer& -          & -          & -          & -             & 72.44 ± 0.28 & 
        - & OOM            & OOM          \\
        NodeFormer & {36.9 ± 1.0} & 38.5 ± 1.5 & 34.7 ± 4.1 & 77.45 ± 1.15  & 59.90 ± 0.42 & -            &  72.93 ± 0.13 &-\\
        CoBFormer & \textbf{\textcolor{deepred}{37.4 ± 1.0}} & - & - & -  & \textbf{\textcolor{orange}{73.17 ± 0.18}} & -            & 78.15 ± 0.07 &-\\
        SGFormer   & \textbf{\textcolor{orange}{37.9 ± 1.1}} & \textbf{\textcolor{orange}{41.8 ± 2.2}} & \textbf{\textcolor{orange}{44.9 ± 3.9}} & \textbf{\textcolor{orange}{79.53 ± 0.38}}  & 72.63 ± 0.13 & -            & 75.36 ± 0.33 &\textbf{\textcolor{orange}{66.01 ± 0.37}} \\
        \midrule 
        GOAT  & 32.1 ± 1.8          & \textbf{\textcolor{deepred}{41.1 ± 2.5}}         & \textbf{\textcolor{deepred}{43.3 ± 4.3}}          & \textbf{\textcolor{deepred}{78.37 ± 0.26}}             & 72.41 ± 0.40 & 
        \textbf{\textcolor{deepred}{53.57 ± 0.18}}            & \textbf{\textcolor{orange}{82.00 ± 0.43}}  & 65.05 ± 0.13          \\
        \textbf{GOAT + HDSE}     & \textbf{\textcolor{darkgreen}{38.0 ± 1.5$^*$}} & \textbf{\textcolor{darkgreen}{43.2 ± 2.4}} & \textbf{\textcolor{darkgreen}{46.0 ± 3.2}} & \textbf{\textcolor{darkgreen}{80.34 ± 0.32$^*$}}  & \textbf{\textcolor{darkgreen}{73.26 ± 0.19$^*$}} & \textbf{\textcolor{darkgreen}{54.23 ± 0.26$^*$}} & \textbf{\textcolor{darkgreen}{83.38 ± 0.17$^*$}} & \textbf{\textcolor{darkgreen}{66.04 ± 0.15$^*$}} \\
        \bottomrule
    \end{tabular}
    \label{tab:tab4}}
     \vspace{-0.2 in}

\end{table*}

\begin{table}
% \vspace{-0.1 in}
\begin{minipage}[t]{0.46\linewidth}
\scriptsize
\caption{Efficiency comparison of GOAT + HDSE and scalable graph transformer competitors; training time per epoch.}
	\centering
	\begin{tabular}{lccc}
		\toprule
		  &PubMed & ogbn-proteins & ogbn-arxiv \\
            \midrule %
             % GOAT & 367.8ms  & 29.4s &  \\
            NodeFormer &321.4ms & 1.8s  & 0.6s \\
            SGFormer &  15.4ms & 0.8s  & 0.2s  \\
           GOAT+HDSE & 13.2ms & 0.6s  & 0.2s  \\
		\bottomrule
	\end{tabular}
 % \vspace{-0.05 in}
     \vspace{-0.1 in}
	\label{tab:tab-ex1}
\end{minipage}
\hfill
\begin{minipage}[t]{0.51\linewidth}
% \vspace{-0.1 in}
\scriptsize
\caption{Ablation study of GOAT + HDSE. "W/o Coarsening" refers to replacing the projection matrix with the original projection matrix used in GOAT.}
	\centering
	\begin{tabular}{lccc}
		\toprule
		  & Actor $\uparrow$ & ogbn-proteins $\uparrow$ & arxiv-year $\uparrow$\\
            \midrule %
             GOAT+HDSE  & 38.0 $\pm$ 1.5  & 80.3 $\pm$ 0.3 &  54.2 $\pm$ 0.2 \\
            w/o HDSE & 34.6 $\pm$ 2.2  & 79.4 $\pm$ 0.3 &  53.6 $\pm$ 0.3 \\
            w/o coarsening & 32.1 $\pm$ 1.8  & 78.3 $\pm$ 0.4 &  53.5 $\pm$ 0.2 \\
		\bottomrule
	\end{tabular}
 \vspace{-0.1 in}
     % \vspace{-0.2 in}
	\label{tab:tab-ex2}
 \end{minipage}

\end{table}

\subsection{Results on Large-scale Graphs}\label{node-level}
\textbf{Overall Performance, Table \ref{tab:tab4}, \ref{tab:tab-cora}.} We select four categories of baselines: GNNs, graph transformers with proven performance on graph-level tasks, graph transformers with hierarchy, and scalable graph transformers. %It's important to note that, 
It is noteworthy that while some graph transformers exhibit superior performance on graph-level tasks, they consistently result in out-of-memory (OOM) in large-scale node tasks. The results are remarkably consistent; Table \ref{tab:tab4} shows our model outperforming all others. In heterophilic graphs (on the left side), the integration of high-level HDSE enables GOAT to demonstrate competitive performance among baseline models. This could be due to the coarsening projection filtering out the edges from neighboring nodes of different categories and providing a global perspective enriched with multi-level structural information. 
For all larger graphs (on the right side), our high-level HDSE method significantly enhances GOAT's performance beyond its original version. This indicates that the structural bias provided by graph hierarchies is capable of handling the node classification task in such larger graphs and effectively retains global information. We investigated this in more detail in our ablation experiments. Furthermore, we also observed that all graph transformers with hierarchy suffer from serious overfitting, attributed to their relatively complex architectures. In contrast, our model's simple architecture contributes to its better generalization. 

\textbf{Efficiency Comparison, Table \ref{tab:tab-ex1}.} We report the efficiency results on PubMed, ogbn-proteins and ogbn-arxiv. It is easy to see that our model outperforms other models in speed, matching the pace of the latest and fastest model, SGFormer \cite{wu2023simplifying}. It achieves true linear complexity with a streamlined architecture.

\textbf{Ablation Study, Table \ref{tab:tab-ex2}.} To determine the utility of our architectural design choices, we conduct ablation experiments on GOAT + HDSE over three datasets. The results presented in Table \ref{tab:tab-ex2}, include (1) removing the high-level HDSE and (2) replacing the coarsening projection matrix with the original projection matrix used in GOAT. These experiments reveal a decline in all performance, thereby validating the effectiveness of our architectural design.

\vspace{-0.1 in}
\section{Conclusions}
We have introduced the Hierarchical Distance Structural Encoding (HDSE) method to enhance the capabilities of transformer architectures in graph learning tasks. We have developed a flexible framework to integrate HDSE with various graph transformers. Further, for applying graph transformers with HDSE to large-scale graphs, we have introduced a high-level HDSE approach that effectively biases linear transformers towards the multi-level structure. Theoretical analysis and empirical results validate the effectiveness and generalization capabilities of HDSE, demonstrating its potential for various real-world applications.

% This work introduces a hierarchy-distance structural encoding as a solution to enhance the transformer architecture's capabilities in graph learning tasks. By considering multi-level graph hierarchical structures and integrating HDSE with various graph transformers, we address the lack of canonical positioning in existing transformers. Our framework allows simultaneous learning of HDSE and other positional encodings. Comprehensive experiments on 18 graph datasets showcase the effectiveness of HDSE in enhancing graph transformers.
% and achieving state-of-the-art performance on 10 benchmark datasets.

% \section*{References}
\bibliography{neurips_2024}

\appendix
\onecolumn

% \setcounter{Proposition}{0} 
% \appendix
\section{Proof}\label{ap-a}
\begin{proposition} \label{th1-app} (Restatement of Proposition \ref{th1})
    GD-WL with HDSE $(\mathrm{D}_{i,j}) $ is strictly more expressive than GD-WL with shortest path distances $(\mathrm{SPD}(i, j))$.
\end{proposition}

\begin{proof}  
First, we show that GD-WL with HDSE is at least as expressive as GD-WL with shortest path distances (SPD). Then, we provide a specific example of two graphs that cannot be distinguished by GD-WL with SPD, but can be distinguished by GD-WL with HDSE.

Let $\mathrm{SPD}(i, j) \in \mathbb{R}$ denote the encoding for shortest path distance. It is worth mentioning that
$$\mathrm{D}_{i,j,0}=\mathrm{GHD}^0\left( i,j \right) = \mathrm{SPD}\left( i,j \right). $$

Thus, $\mathrm{D}_{i,j}$ is a function of $\mathrm{SPD}\left( i,j \right)$, and hence $\mathrm{D}_{i,j}$ refines $\mathrm{SPD}\left( i,j \right)$. To conclude this, we utilize Lemma 2 from \cite{bevilacqua2021equivariant}, which states that refinement is maintained when using multisets of colors. This observation confirms that GD-WL with HDSE is at least as powerful as GD-WL with SPD.

To show that GD-WL with HDSE is strictly more expressive, we provide an example of two non-isomorphic graphs that can be distinguished by the HDSE but not the SPD: the Desargues graph and the Dodecahedral graph. As depicted in Figure 6 of \cite{zhang2023rethinking}, it has been observed that GD-WL with SPD fails to distinguish these graphs. However, GD-WL with our HDSE can. Figure \ref{fig:HDSE2} shows the coarsening results of the Girvan-Newman Algorithm \cite{girvan2002community}. We can demonstrate that, for the Dodecahedral graph, each node has $1$-level hierarchy distances of length 2 to other nodes. On the other hand, in the Desargues graph, there are no such distances of length 2 between any pair of nodes.
\end{proof}

\begin{figure}[htb]   \center{\includegraphics[width=8cm]  {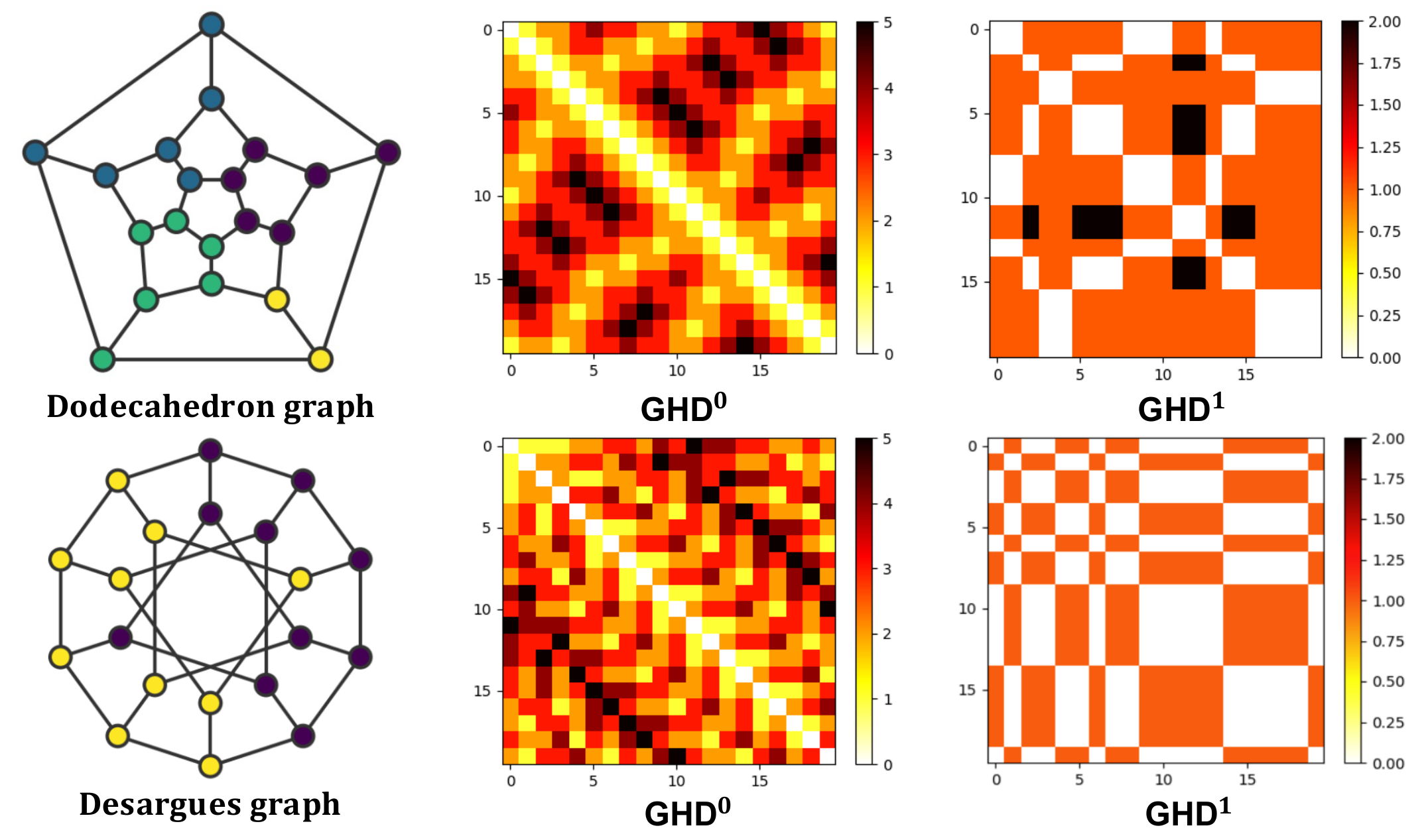}}   \caption{GD-WL with HDSE can distinguish Dodecahedron and Desargues graphs, but GD-WL with SPD cannot.}   \label{fig:HDSE2} \end{figure}

\begin{proposition}\label{th2-1-app} (Restatement of Corollary \ref{th2-1}) The power of a graph transformer with HDSE to distinguish non-isomorphic graphs is at most equivalent to that of the GD-WL test with HDSE. With proper parameters and an adequate number of heads and layers, a graph transformer with HDSE can match the power of the GD-WL test with HDSE.
\end{proposition}

\begin{proof}  
The theorem is divided into two parts: the first and second halves. We begin by considering the first half: The power of a graph transformer with HDSE to distinguish non-isomorphic graphs is at most equivalent to that of the GD-WL test with HDSE.

Recall that the GD-WL with HDSE is quite straightforward and can be expressed as:
$$
\chi^t_G(v) := \text{hash} \left\{ {(\mathrm{D}_{v, u}, \chi^{t-1}_G (u)) : u \in V} \right\}
$$
where $\chi^t_G(v)$ represents a color mapping function.

Suppose after \( t \) iterations, a graph transformer with HDSE \(\mathcal{M}\) has \(\mathcal{M}(G_1) \neq \mathcal{M}(G_2)\), yet GD-WL with HDSE fails to distinguish \( G_1 \) and \( G_2 \) as non-isomorphic. It implies that from iteration 0 to \( t \) in the GD-WL test, \( G_1 \) and \( G_2 \) always have the same collection of node labels. Particularly, since \( G_1 \) and \( G_2 \) have the same GD-WL node labels for iterations \( i + 1 \) for any \( i = 0, \ldots, t - 1 \), they also share the same collection of GD-WL node labels \( \left\{ (\mathrm{D}_{v, u}, \chi^{i}_G (u)) : u \in V \right\}\). Otherwise, the GD-WL test would have produced different node labels at iteration \( i + 1 \) for \( G_1 \) and \( G_2 \).

We show that within the same graph, for example \( G = G_1 \), if GD-WL node labels \(\chi^i_G(v) = \chi^i_G(w)\), then the graph transformer node features \(h^i_v = h^i_w \) for any iteration \( i \). This is clearly true for \( i = 0 \) because GD-WL and graph transformer start with identical node features. Assuming this holds true for iteration \( j \), if for any \( v, w \), \( \chi^{j+1}_G(v) = \chi^{j+1}_G(w) \), then we must have 
$$\left\{ (\mathrm{D}_{v, u}, \chi^{j}_G (u)) : u \in V \right\} = \left\{ (\mathrm{D}_{w, u}, \chi^{j}_G (u)) : u \in V \right\}.$$
By our assumption at iteration \( j \), we deduce that
$$h^{j+1}_v ={\sum_{u \in V} \mathrm{softmax}(\mathrm{Bias}(\mathrm{D}_{v, u})+ h^j_v\mathbf{W}_{\mathbf{Q}}(h^j_u\mathbf{W}_{\mathbf{K}})^{\top})h_u^j\mathbf{W}_{\mathbf{V}}} =\phi(\left\{ {(\mathrm{D}_{v, u}, \chi^{j}_G (u)) : u \in V} \right\}).$$
Hence,
$$h^{j+1}_v = \phi(\left\{ {(\mathrm{D}_{v, u}, \chi^{j}_G (u)) : u \in V} \right\}) = \phi(\left\{ {(\mathrm{D}_{w, u}, \chi^{j}_G (u)) : u \in V} \right\}) = h^{j+1}_w.$$
By induction, if GD-WL node labels \(\chi^i_G(v) = \chi^i_G(w)\), we always have the graph transformer node features \(h^i_v = h^i_w \) for any iteration \( i \). Consequently, from \( G_1 \) and \( G_2 \) having identical GD-WL node labels, it follows that they also have the same graph transformer node features.

Therefore, \(h^{i+1}_v = h^{i+1}_w\). Given that the graph-level readout function is permutation-invariant with respect to the collection of node features, \(\mathcal{M}(G_1) = \mathcal{M}(G_2)\). This leads to a contradiction.

This completes the proof of the first half of the theorem. For the theorem's second half, we can entirely leverage the proof of Theorem E.3 by \cite{zhang2023rethinking} (provided in Appendix E.3), which presents a similar situation.

\end{proof}

\begin{proposition}\label{th3-ap} (Full version of Proposition \ref{th3})
    For a semi-supervised binary node classification problem, suppose the label of each node $i\in V$ in the whole graph is determined by the majority vote of discriminative node features in the ``hierarchical core neighborhood'': $S_*^i=\{j: \mathrm{D}_{i,j}=D^*\}$ for a certain $D^*$, where $\mathrm{D}_{i,j}$ is HDSE defined in (\ref{eqn: HDSE D_ij}). Assume noiseless node features. Then, as long as the model is large enough, the batch size $B\geq \Omega(\epsilon^{-2})$, the step size $\eta<1$, 
 the number of iterations $T$ satisfies $T=\Theta(\eta^{-1/2})$ and the number of known labels satisfies $
    |\mathcal{L}|\geq \max\{\Omega((1+\delta_{D^*}^2)\cdot\log N), BT\}$, 
where $\delta_{D^*}$ measures the maximum number of nodes in the hierarchical core neighborhood $S_*^n$ for all nodes $n$, then with a probability of at least $0.99$, the returned one-layer graph transformer with HDSE trained by the SGD variant Algorithm 1 and Hinge loss in \citep{li2023what} can achieve a generalization error which is at most $\epsilon$ for any $\epsilon>0$. However, we do not have a generalization guarantee to learn such a graph characterized by the hierarchical core neighborhood with a one-layer graph transformer with SPD positional encoding. 
\end{proposition}

Before starting the proof, we first briefly introduce and extend some notions and setups used in \citep{li2023what}. \textbf{The major differences are that (1) we extend their core neighborhood from based on SPD to HDSE (2) we use HDSE in the transformer by encoding it as a one-hot encoding for simplicity of analysis.}

Their work focuses on a semi-supervised binary node classification problem on structured graph data, where each node feature corresponds to either a discriminative or a non-discriminative feature, and the dominant discriminative feature in the core neighborhood determines each ground truth node label. For each node, the neighboring nodes with features consistent with the label are called class-relevant nodes, while nodes with features opposite to the label are called confusion nodes. Denote the class-relevant and confusion nodes set for node $n$ as $\mathcal{D}_*^n$ and $\mathcal{D}_\#^n$, respectively. A new definition here is the distance-$D$ neighborhood $\mathcal{N}_{D}^n$, which is the set of nodes $\{j: \mathrm{D}_{n,j}=D, j\in V\}$. $\mathrm{D}$ is the HDSE defined in (\ref{eqn: HDSE D_ij}). Then, by following Definition 1 in \citep{li2023what}, we define the winning margin for each node $n$ of distance $D$ as $\Delta_n(D)=|\mathcal{D}_*^n\cap\mathcal{N}_D^n|-|\mathcal{D}_\#^n\cap\mathcal{N}_D^n|$. The core distance $D^*$ is the distance $D$ where the average winning margin over all nodes is the largest. We call the set of neighboring nodes $S_*^n=\{j: \mathrm{D}_{n,j}=D^*\}$ the core neighborhood. We then make the assumption that $\Delta_n(D^*)>0$ for all nodes $n\in V$, following Assumption 1 in \citep{li2023what}. The one-layer transformer we study is formulated as
\begin{equation}
F(h_n)=\text{Relu}(\sum_{u \in V} \mathrm{softmax}(\mathrm{B}(\mathrm{D}_{n, u})^\top b+ h_n\mathbf{W}_{\mathbf{Q}}(h_u\mathbf{W}_{\mathbf{K}})^{\top})h_u\mathbf{W}_{\mathbf{V}} \mathbf{W}_{\mathbf{O}} )\boldsymbol{a}\label{arch one-layer}
\end{equation}
where $\boldsymbol{W}_O\in\mathbb{R}^{d'\times d''}$ and $\boldsymbol{a}\in\mathbb{R}^{d''}$ are the hidden and output weights in the two-layer feedforward network, and $b\in\mathbb{R}^Z$ is the trainable parameter to learn the positional encoding. The one-hot relative positional encoding $\mathrm{B}(\mathrm{D}_{n, u})$ is defined as
\begin{equation}
    \mathrm{B}(\mathrm{D}_{n, u})=\boldsymbol{c}_s,\label{B one hot}
\end{equation}
where $\boldsymbol{c}_s$ is the $s$-th standard basis in $\mathbb{R}^{Z}$. $Z$ is the total number of all possible values of $\mathrm{D}_{n,u}$ for $n,u\in V$. %Denote the total number of all possible values of $\mathrm{D}_{n,u}$ for $n,u\in V$ (which is assumed to be finite) as $|D_V|$. 
%$\text{Map}:\mathbb{R}^{K+1}\mapsto \{0,1,2,\cdots, Z\}$ is a bijection from a HDSE vector to a scalar. \text{if } \text{Map}(\mathrm{D}_{n,u})=s-1\leq Z-1
$\mathrm{B}(\cdot)$ is a bijection from $\{d\in\mathbb{R}^{K+1}: d=\mathrm{D}_{n,u}\text{, for certain }n,v\in V\}$ to $\{\boldsymbol{c}_1,\boldsymbol{c}_2,\cdots,\boldsymbol{c}_{Z}\}$.

\iffalse
The one-hot relative positional encoding $\mathrm{B}(\mathrm{D}_{v, u})$ is defined as
\begin{equation}
    \mathrm{B}(\mathrm{D}_{v, u})=\begin{cases}\boldsymbol{e}_s, &\text{if } \text{Map}(\mathrm{D}_{v,u})=s-1\leq Z-1,
    \\\boldsymbol{e}_{Z}, &\text{if }\text{Map}(\mathrm{D}_{v,u})=s-1> Z-1,\end{cases}\label{B one hot}
\end{equation}
where $\boldsymbol{e}_s$ is the $s$-th standard basis in $\mathbb{R}^{Z}$. Denote the total number of all possible values of $\mathrm{D}_{v,u}$ for $v,u\in V$ (which is assumed to be finite) as $|D_V|$. $\text{Map}:\mathbb{R}^{K+1}\mapsto \{0,1,2,\cdots, |D_V|\}$ is a bijection from a HDSE vector to a scalar. 
\fi

Then, we provide the proof for Proposition \ref{th3-ap}.

\begin{proof}
The proof follows Theorem 4.1 in \citep{li2023what} given the above reformulation. Note that (\ref{B one hot}) can also map the SPD relationship, which is a special one-dimensional case of $\mathrm{D}_{v,u}$, between nodes as (2) in \citep{li2023what} by the definition of itself. It means that (\ref{arch one-layer}) with HDSE can achieve a generalization performance on the graph characterized by the core neighborhood as good as in \citep{li2023what}. 

However, we cannot have an inverse conclusion, i.e., providing a generalization guarantee on the graph characterized by the hierarchical core neighborhood using (\ref{arch one-layer}) with SPD. This is because SPD cannot distinguish nodes with the same SPD but different HDSE to a certain node. Hence, for a certain node $i\in V$, aggregating nodes using SPD may include nodes outside the hierarchical core neighborhood, of which the labels are inconsistent with the node $i$, and lead to a wrong prediction.   %When $\boldsymbol{b}=0$ is fixed during the training, but the nodes used for training and testing in aggregation for node $n$ are subsets of $\mathcal{N}_{D^*}^n$, the bound for attention weights on class-relevant nodes is still the same as in (63) and (64) of \citep{li2023what}. Given a known core neighborhood $S_*^n$, the remaining parameters follow the same order-wise update as Lemmas 4, 5, and 7. The remaining proof steps follow the contents in the proof of Theorem 4.1 of \citep{li2023what}, which leads to a generalization on a one-layer transformer with HDSE and aggregation with all nodes in the graph. 
\end{proof}

\begin{proposition}
    For a semi-supervised binary node classification problem, suppose the label of each node $i\in V$ in the whole graph is determined by the majority vote of discriminative node features in the ``core neighborhood'': $S_*^i=\{j: \mathrm{D}_{i,j}=D^*\}$ for a certain $D^*$, where $\mathrm{D}_{i,j}$ is HDSE defined in (\ref{eqn: HDSE D_ij}). Then, for each node $i\in V$, a properly initialized one-layer graph transformer (i) \textbf{without HDSE} (ii) and \textbf{only aggregate nodes from $S_*^i$} can achieve the same generalization error as learning with a one-layer graph transformer (a) \textbf{with HDSE} (b) \textbf{aggregate all nodes in the graph without knowing $S_i^*$ in prior}. 
\end{proposition}

\begin{proof}
The proof follows Theorem 4.3 in \citep{li2023what}. When $\boldsymbol{b}=0$ is fixed during the training, but the nodes used for training and testing in aggregation for node $n$ are subsets of $\mathcal{N}_{D^*}^n$, the bound for attention weights on class-relevant nodes is still the same as in (63) and (64) of \citep{li2023what}. Given a known core neighborhood $S_*^n$, the remaining parameters follow the same order-wise update as Lemmas 4, 5, and 7. The remaining proof steps follow the contents in the proof of Theorem 4.1 of \citep{li2023what}, which leads to a generalization on a one-layer transformer with HDSE and aggregation with all nodes in the graph. 

\end{proof}

\newpage
\section{Experimental Details}\label{ap-b}
\subsection{Computing Environment} Our implementation is based on PyG \citep{fey2019fast} and DGL \cite{wang2019deep}. The experiments are conducted on a single workstation with 4 RTX 3090 GPUs and a quad-core CPU.

\begin{table*}[h]
	\centering
        \scriptsize
         \caption{Overview of the graph learning dataset used in this work \cite{dwivedi2023benchmarking,dwivedi2022long,kipf2017semisupervised,chien2020adaptive,pei2019geom,rozemberczki2021multi,hu2020open,mcauley2015inferring,leskovec2016snap}.}
	\begin{tabular}{lcccccccc}
		\toprule
		{Dataset} & {\# Graphs} & {Avg. \# nodes} & {Avg. \# edges} & {\# Feats} & {Prediction level} & {Prediction task} & {Metric}\\
		% \cmidrule{2-3,4-5,6-7}
		\midrule %
        ZINC & 12,000 & 23.2 & 24.9 &28 & graph & regression & MAE \\
        MNIST & 70,000 & 70.6 & 564.5 &3 & graph & 10-class classif. & Accuracy \\
        CIFAR10 & 60,000 & 117.6 & 941.1 &5& graph & 10-class classif. & Accuracy \\
        PATTERN & 14,000 & 118.9 & 3,039.3 &3& node & binary classif. & Accuracy \\
        CLUSTER & 12,000 & 117.2 & 2,150.9 &7&node & 6-class classif. & Accuracy \\
        \midrule %
        Peptides-func & 15,535 & 150.9 & 307.3 &9 & graph & 10-task classif. & AP \\
        Peptides-struct & 15,535 & 150.9 & 307.3 &9& graph & 11-task regression & MAE \\
        \midrule %
        Cora & 1 & 2,708 & 5,278 &2,708 & node & 7-class classif. & Accuracy \\
        Citeseer & 1 & 3,327 &4,522 &3,703 & node & 6-class classif. & Accuracy \\
        Pubmed & 1 & 19,717 & 44,324 &500 & node & 3-class classif. & Accuracy \\
        Actor & 1 & 7,600 & 26,659 &931 & node & 5-class classif. & Accuracy \\
        Squirrel & 1 & 5,201 & 216,933 &2,089 & node 
 & 5-class classif. & Accuracy \\
        Chameleon & 1 & 2,277 &36,101 &2,325 & node 
 & 5-class classif. & Accuracy \\
         \midrule %
         ogbn-proteins &1 &132,534 &39,561,252 & 8 &node &112 binary classif. & ROC-AUC\\ 
         ogbn-arxiv &1 &169,343 &1,166,243 &128 &node &40-class classif. & Accuracy\\ 
         arxiv-year &1 &169,343 &1,166,243 &128 &node &5-class classif. & Accuracy\\
         ogbn-products &2 & 2,449,029 & 61,859,140 &100 &node &47-class classif. & Accuracy\\
         ogbn-papers100M &1 &111,059,956 &1,615,685,872 &128 &node &172-class classif. & Accuracy\\
        % Wisconsin & 1 & 251 & 499 & transductive node & 5-class classif. & Accuracy \\

        \bottomrule
	\end{tabular}
	\label{tab:dataset}
\end{table*}

\subsection{Description of Datasets}
Table \ref{tab:dataset} presents a summary of the statistics and characteristics of the datasets. The initial five datasets are sourced from \cite{dwivedi2023benchmarking}, followed by two from \cite{dwivedi2022long}, and finally the remaining datasets are obtained from \cite{kipf2017semisupervised,chien2020adaptive,pei2019geom,rozemberczki2021multi,hu2020open,mcauley2015inferring,leskovec2016snap}. 
% For more comprehensive details of the datasets, readers are encouraged to refer to \cite{rampavsek2022recipe} and \cite{wu2023simplifying}.
\begin{itemize}[leftmargin=*] % ,noitemsep,topsep=0pt
    \item ZINC, MNIST, CIFAR10, PATTERN, CLUSTER, Peptides-func and Peptides-struct. For each dataset, we follow the standard train/validation/test splits and evaluation metrics in \cite{rampavsek2022recipe}. For more comprehensive details, readers are encouraged to refer to \cite{rampavsek2022recipe}.
    \item Cora, Citeseer, Pubmed, Actor, Squirrel, Chameleon, ogbn-proteins, ogbn-arxiv, ogbn-products and ogbn-papers100M. For each dataset, we use the same train/validation/test splits and evaluation metrics as \cite{wu2023simplifying}. For detailed information on these datasets, please refer to \cite{wu2023simplifying}.
    \item Arxiv-year is a citation network among all computer science arxiv papers, as described by \cite{lim2021large}. In this network, each node corresponds to an arxiv paper, and the edges indicate the citations between papers. Each paper is associated with a 128-dimensional feature vector, obtained by averaging the word embeddings of its title and abstract. The word embeddings are generated using the WORD2VEC model. The labels of arxiv-year are publication years clustered into fve intervals. We use the public splits shared by \cite{lim2021large}, with a train/validation/test split ratio of 50\%/25\%/25\%.
\end{itemize}

\begin{table*}[h]
        % \vspace{-0.2 in}
	\centering
 \footnotesize
         \caption{Hyperparameters of GraphGPS + HDSE for five datasets from \cite{dwivedi2023benchmarking}.}
	\begin{tabular}{lcccccc}
		\toprule
		{Hyperparameter} & {ZINC} & {MNIST} & {CIFAR10} & {PATTERN} & {CLUSTER}\\
		% \cmidrule{2-3,4-5,6-7}
		\midrule %
            \# GPS Layers & 10 & 3 & 3 & 6 & 16\\
            Hidden dim & 64 & 52 & 52 & 64 & 48\\
            GPS-MPNN & GINE & GatedGCN & GatedGCN & GatedGCN & GatedGCN\\
            GPS-GlobAttn & Transformer & Transformer & Transformer & Transformer & Transformer\\
            \# Heads & 4 & 4 & 4 & 4 & 8 \\
            Attention dropout & 0.5 & 0.5 & 0.5 & 0.5 & 0.5\\
            Graph pooling & sum & mean & mean & – & –\\
            \midrule %
            Positional Encoding & RWSE-20 & LapPE-8 & LapPE-8 & LapPE-16 & LapPE-10\\
            PE dim & 28 & 8 & 8 & 16 & 16\\
            PE encoder & linear & DeepSet & DeepSet & DeepSet & DeepSet\\
            Batch size & 32 & 16 & 16 & 32 & 16\\
            \midrule %
            Learning Rate & 0.001 & 0.001 & 0.001 & 0.0005 & 0.0005\\
            \# Epochs & 2000 & 100 & 200 & 100 & 100\\
            \# Warmup epochs & 50 & 5 & 5 & 5 & 5\\
            Weight decay & 1e-5 & 1e-5 & 1e-5 & 1e-5 & 1e-5\\
            \midrule %
            $K$ & 1 & 1 & 1 & 1 & 1\\
            Coarsening algorithm & Newman & Louvain & Louvain & Loukas ($\alpha=0.1$) & Louvain\\
            \midrule %
            \# Parameters & 437,389 & 124,565& 121,913 & 352,695 & 517,446\\
            % Time (epoch) & 23s & 78s & 71s & 61s & 91s\\
            % \vspace{-0.2 in}
        \bottomrule
	\end{tabular}
  % \vspace{-0.1 in}
 % \vspace{-0.05 in}
	\label{tab:dataset1}
\end{table*}

\begin{table}[h]
        \caption{Hyperparameters of GraphGPS + HDSE for two LRGB datasets from \cite{dwivedi2022long}.}
        \vspace{0.1 in}
	\centering
 \footnotesize
	\begin{tabular}{lcc}
		\toprule
		{Hyperparameter} &Peptides-func & Peptides-struct\\
		% \cmidrule{2-3,4-5,6-7}
		\midrule %
            \# GPS Layers & 4 & 4\\
            Hidden dim & 96 & 96 \\
            GPS-MPNN & GatedGCN & GatedGCN\\
            GPS-GlobAttn & Transformer & Transformer\\
            \# Heads & 4 & 4 \\
            Attention dropout & 0.5 & 0.5\\
            Graph pooling & mean & mean \\
            \midrule %
            Positional Encoding & LapPE-10 & LapPE-10 \\
            PE dim & 16 & 16 \\
            PE encoder & DeepSet & DeepSet \\
            Batch size & 16 & 128 \\
            \midrule %
            Learning Rate & 0.0003 & 0.0003 \\
            \# Epochs & 200 & 200 \\
            \# Warmup epochs & 5 & 5 \\
            Weight decay & 0 & 0 \\
            \midrule %
            $K$ & 0 & 1 \\
            Coarsening algorithm & Newman & METIS ($\alpha=0.1$) \\
            \midrule %
            \# Parameters & 505,866 & 506,235 \\
            % Time (epoch) & 39.5s & 39.1s\\
        \bottomrule
	\end{tabular}
	\label{tab:dataset2}
\end{table}

% \subsection{Dataset splits and random seeds} The experiments are conducted using the standard train/validation/test splits of the evaluated benchmarks. All main benchmarking results are based on 5 executed runs with different random seeds (0, 1, 2, 3, 4), except large-scale datasets. For large-scale datasets, We apply 3 independent runs on random data splitting and report the means and standard deviations.
% \newpage
\subsection{Hyperparameter and Reproducibility} 
\textbf{Models + HDSE.}
For fair comparisons, we use the same hyperparameters (including training schemes, optimizer, number of layers, batch size, hidden dimension etc.) as baseline models for all of our HDSE versions. Taking GraphGPS + HDSE as an example, Tables \ref{tab:dataset1} and \ref{tab:dataset2} showcase the corresponding hyperparameters and coarsening algorithms. It is important to note that our HDSE module introduces only a small number of additional parameters. And each experiment is repeated 5 times to get the mean value and error bar.
% \newpage
% Furthermore, only for ANS-GT + HDSE, we tune the hyperparameters on the large-scale datasets using the same grid search as reported in their paper \cite{zhang2022hierarchical}.

\textbf{GOAT + HDSE.}
To accelerate training, we do not adopt the neighbor sampling (NS) method from GOAT to sample neighbors; instead, we train directly on the entire graph. For graphs with over one million nodes, we randomly sample nodes within the graph and select their induced subgraph for batch training. For the hyperparameter selections of our high-level HDSE model, in addition to what we have covered in the setting part of the experiment section that datasets share in common, we list other settings in Table~\ref{tab:HDSE-parameter}. It's important to note that our hyperparameters were determined within the SGFormer's grid search space. Furthermore, all other experimental parameters, including dropout, batch size, training schemes, optimizer, etc., are consistent with those used in the SGFormer \cite{wu2019simplifying}. The testing accuracy achieved by the model that reports the highest result on the validation set is used for evaluation. And each experiment is repeated 10 times to get the mean value and error bar.

\textbf{SGFormer on Graph-level Tasks.} To accurately demonstrate the capabilities of SGFormer on these datasets, we use all the same experimental settings and conduct the same grid search as outlined in GraphGPS \cite{rampavsek2022recipe}.

\begin{table*}[h]
	\centering
        \scriptsize
         \caption{GOAT + HDSE dataset-specific hyperparameter settings.}
	\begin{tabular}{lcccccccccc}
		\toprule
		{Dataset} &$K$ & $|V^{1}|$ & Hidden dim & \# Heads & \# Glob. Layers & Local GNN & \# GNN Layers & \# Epochs  & LR \\
		% \cmidrule{2-3,4-5,6-7}
        \midrule %
        Cora & 1& 32 & 128 & 4 & 1 & GCN & 2  &500 & 1e-2\\
        Citeseer & 1& 200 &128 &2 &1 & GCN &2  &500 &1e-2 \\
        Pubmed & 1& 64 & 128 &1 &1 &GCN &2  &500 &1e-2 \\
        Actor & 1& 200 &128 &2 &1 &GCN &2 &1000 &1e-2  \\
        Squirrel & 1& 128 & 128 &1 &3 &GCN &2  &500 &1e-2 \\
        Chameleon & 1& 32 &128 &1 &3 & GCN &2 &500 &1e-2 \\
         \midrule %
         ogbn-proteins & 1& 1024 & 128 &2 &1 & GraphSAGE &4  &1000 &5e-4\\ 
         ogbn-arxiv & 1& 1024 &256 & 1 &1 & GCN &7  &2000&5e-4\\ 
         arxiv-year & 1& 2048 & 128&4&1&GAT&1 &500&1e-3\\
         ogbn-products & 2 &1024 &256 & 4 &1 & GraphSAGE &5  &1000&3e-3\\
         ogbn-100M & 1&1024 &256 & 1 &1 & GCN &3 &50&1e-3 \\
        % Wisconsin & 1 & 251 & 499 & transductive node & 5-class classif. & Accuracy \\

        \bottomrule
	\end{tabular}
	\label{tab:HDSE-parameter}
\end{table*}

\newpage
\section{Additional Experimental Results}\label{ap-c}
\subsection{Cora, Citeseer and Pubmed} 

\begin{table*}[t]
    \centering
    {
    \caption{\textbf{Node classification} on Cora, Citeseer and Pubmed (\%).}
    % \vspace{-0.1 in}
    \begin{tabular}{c|cccccc|cccccc}
        \toprule
        Model    & Cora       & CiteSeer   & PubMed     \\
         \midrule 
        \# nodes    & 2,708      & 3,327      & 19,717     \\
        \# edges & 5,278 & 4,552 & 44,324  \\
         & Accuracy$\uparrow$ & Accuracy$\uparrow$ & Accuracy$\uparrow$\\
        \midrule %
        GCN        & 81.6 ± 0.4 & 71.6 ± 0.4 & 78.8 ± 0.6 \\
        GAT        & 83.0 ± 0.7 & 72.1 ± 1.1 & 79.0 ± 0.4  \\
        SGC        & 80.1 ± 0.2 & 71.9 ± 0.1 & 78.7 ± 0.1 \\
        JKNet &81.8 ± 0.5 &70.7 ± 0.7 &78.8 ± 0.7  \\
        APPNP &83.3 ± 0.5 &71.8 ± 0.5 &80.1 ± 0.2 \\
        SIGN &82.1 ± 0.3 &72.4 ± 0.8 &79.5 ± 0.5 \\
        HC-GNN & 81.9 ± 0.4          & 72.5 ± 0.6          & 80.2 ± 0.6           \\
        \midrule 
        Graphormer &75.8 ± 1.1 &65.6 ± 0.6 &OOM \\
        SAT & 72.4 ± 0.3 & 60.9 ± 1.3 & OOM \\
        \midrule 
        ANS-GT &79.4 ±  0.9 &64.5 ±  0.7 &77.8 ±  0.7 \\
        AGT &81.7 ± 0.4 &71.0 ± 0.6 & - \\
        HSGT &83.6 ±  1.8 &67.4 ±  0.9 &79.7 ±  0.5 \\
        \midrule 
        GraphGPS & 76.5 ± 0.6 & - & 65.7 ± 1.0 \\
        Gapformer       & 83.5 ± 0.4          & 71.4 ± 0.6          & 80.2 ± 0.4            \\
        NodeFormer & 82.2 ± 0.9 & 72.5 ± 1.1 & 79.9 ± 1.0 \\
        SGFormer   & 84.5 ± 0.8 & 72.6 ± 0.2 & 80.3 ± 0.6   \\
        \midrule 
        GOAT       &82.1 ± 0.9  & 71.6 ± 1.3 & 78.9 ± 1.5 \\
        \textbf{GOAT + HDSE}       &83.9 ± 0.7  & 73.1 ± 0.7 & 80.6 ± 1.0   \\
        \bottomrule
    \end{tabular}
    \label{tab:tab-cora}}

\end{table*}

\subsection{Sensitivity Analysis of Maximal Hierarchy Level $K$}

\begin{table}[h]
% \vspace{-0.05 in}
\footnotesize
\caption{Sensitivity analysis on the maximum hierarchy level $K$ of GraphGPS + HDSE on ZINC.}
	\centering
	\begin{tabular}{lccc}
		\toprule
		  & $K=0$ (SPD) & $K=1$ & $K=2$\\
            \midrule %
             ZINC  $\downarrow$ & 0.069 $\pm$ \footnotesize{0.003}  & 0.062 $\pm$ \footnotesize{0.003} &  0.064 $\pm$ \footnotesize{0.004} \\
		\bottomrule
	\end{tabular}
 % \vspace{-0.05 in}
     % \vspace{-0.2 in}
	\label{tab:tab6}
\end{table}

\begin{table}[h]
% \vspace{-0.05 in}
\footnotesize
\caption{Sensitivity analysis on the maximum hierarchy level $K$ of GOAT + HDSE.}
	\centering
	\begin{tabular}{cccc}
		\toprule
		  &Squirrel $\uparrow$ & arxiv-year $\uparrow$ & ogbn-arxiv $\uparrow$\\
            \midrule %
              $K=1$ & 43.2 $\pm$ 2.4  & 54.23 $\pm$ 0.26 & 73.26 $\pm$ 0.19 \\
              $K=2$ & 43.8 $\pm$ 2.1  & 54.51 $\pm$ 0.17 & 73.36 $\pm$ 0.15 \\
		\bottomrule
	\end{tabular}
	\label{tab:tab7}
\end{table}

We conduct a sensitivity analysis on the maximum hierarchy level $K$ on ZINC. The results are shown in Table \ref{tab:tab6}. Note that when $K=0$, HDSE degenerates into SPD, leading to a worse performance. This result of $K$ is about graph classification tasks, where the size of graphs is typically small. Therefore, at level 1 ($K=1$), the quantity of coarsened nodes is quite small, eliminating the necessity for a higher $K$. We further investigate the impact of $K$ on large-graph node classification, across 3 datasets: Squirrel, arxiv-year, and ogbn-arxiv. Based on the results displayed in Table \ref{tab:tab7}, we can make the following observations: (1) $K=1$ does not consistently yield the best results. Optimal performance is achieved with $K=2$ on some datasets. (2) The improvement brought about by $K=2$ over $K=1$ is relatively minor. 

The variation in the optimal $K$ could stem from the distinct hierarchical structures inherent in different graphs. Larger graphs may possess more pronounced multi-level structures, thus necessitating a higher $K$. However, the slight improvement resulting from a larger $K$ could suggest limitations in the coarsening algorithm.

This study reinforces our selection of $K=1$, aligning with results from other hierarchical graph transformer papers such as HSGT \cite{zhu2023hierarchical} and ANS-GT \cite{zhang2022hierarchical}. We anticipate that the real potential of a higher $K$ will be revealed through the application of a proper, effective coarsening algorithm on graphs with hierarchical community structures. We look forward to exploring this in the future.

\subsection{Sensitivity Analysis of Maximum Distance Length $L$}

\begin{table}[h]
\scriptsize
\centering
\caption{Overview of the graph diameters of datasets used in graph classification}
\label{diameter}
\begin{tabular}{cccccccc}
		\toprule
		   & ZINC & MNIST & CIFAR10 & PATTERN & CLUSTER & Peptides-func & Peptides-struct \\
            \midrule %
              Average Diameter & 12.47 & 6.85 & 9.17 & 2.00 & 2.17 & 56.86 & 56.86 \\ 
                Maximum Diameter & 22 & 8 & 12 & 3 & 3 & 159 & 159 \\
		\bottomrule
	\end{tabular}
 
\end{table}

For each graph classification dataset, we calculate the graph diameter of each graph in the dataset and then compute the average graph diameter and maximum graph diameter for the entire dataset, as detailed in Table \ref{diameter}. Note that we use high-level HDSE to deal with node classification on large graphs; therefore, we do not calculate the distances between the nodes in these large graphs. The data indicates $L=30$ is a reasonable choice, as it encompasses most of the graph diameters. We do not use a larger number as we hypothesize that for graphs with larger diameters, the utility of detailed information loses significance beyond a certain distance.

Additionally, we conducted a sensitivity analysis regarding the selection of $L$, as outlined in Table \ref{tab:tab8}, which confirms that $L=30$ is an appropriate choice.

\begin{table}[h]
% \vspace{-0.05 in}
\footnotesize
\caption{Sensitivity analysis on the maximum distance length $L$.}
	\centering
	\begin{tabular}{cccc}
		\toprule
		  &Peptides-func $\uparrow$ & Peptides-struct $\uparrow$ \\
            \midrule %
              GraphGPS + HDSE $(L=20)$ & 0.7105 $\pm$ 0.0051  & 0.2481 $\pm$ 0.0016  \\
              GraphGPS + HDSE $(L=30)$ & 0.7156 $\pm$ 0.0058  & 0.2457 $\pm$ 0.0013  \\
              GraphGPS + HDSE $(L=50)$ & 0.7124 $\pm$ 0.0053  & 0.2466 $\pm$ 0.0021  \\
		\bottomrule
	\end{tabular}
	\label{tab:tab8}
\end{table}

\subsection{Coarsening Runtime} 
Table \ref{tab:runtime-e} gives the runtime of coarsening algorithms (including distance calculation) on graph-level tasks, illustrating the practicality of our method. The Newman algorithm is unsuited for larger graphs due to high complexity. In addition, our HDSE module almost does not increase the runtime of the baselines. For example, GraphGPS runs at 10 seconds per epoch, compared to 11 seconds per epoch with HDSE module on ZINC.

Additionally, for all large-scale graphs, we employ METIS due to its efficiency with a time complexity of $O(|E|)$. This makes it highly effective for partitioning extensive graphs, such as ogbn-products, in less than 5 minutes, and even the vast ogbn-papers100M, with a size of 0.1 billion nodes, requires only 59 minutes.

\begin{table}[h]
	\centering
	\caption{Empirical runtime of coarsening algorithms.}
	\begin{tabular}{ccccc}
		\toprule
		 Algorithm & ZINC & PATTERN &MNIST &P-func\\
		\midrule %
		METIS  & 31s & 0.1h & 0.2h & 0.1h\\
            Newman  & 88s & $>$500h & 18h & 1.6h\\
		Louvain & 76s & 5h & 1.6h & 1.1h\\
		\bottomrule
	\end{tabular}
	\label{tab:runtime-e}
\end{table}

\begin{table}[h]
% \vspace{-0.2 in}
% \footnotesize
\caption{Node classification on synthetic community datasets.}
	\centering
	\begin{tabular}{lccc}
		\toprule
		  Dataset & GT & GT + SPD & GT + HDSE\\
            \midrule %
             Community-small & 64.7 $\pm$ 1.1  & 81.5 $\pm$ 1.7 &  88.6 $\pm$ 0.9 \\
		\bottomrule
	\end{tabular}
 % \vspace{-0.05 in}
  % \vspace{-0.2 in}
     % \vspace{-0.1 in}
	\label{tab:syn}
\end{table}

\subsection{Synthetic Community Dataset} 
% a synthetic dataset \textbf{Community-small} \cite{you2018graphrnn}
We evaluate the \textbf{Community-small} dataset from GraphRNN \cite{you2018graphrnn}, a synthetic dataset featuring community structures. It comprises 100 graphs, each with two distinct communities. These communities are generated using the Erdos-Renyi model (E-R). Node features are generated from random numbers and node labels are determined by their respective cluster numbers with accuracy as the chosen evaluation metric. We use the a random train/validation/test split ratio of
60\%/20\%/20\%. 

We select the Louvain method as our coarsening algorithm and integrate the HDSE module into the Graph Transformer (GT). As shown in Table \ref{tab:syn}, the GT struggles to detect such structures; and solely utilizing SPD proves inadequate; however, our HDSE, enriched with coarsening structural information, effectively captures these structures.

\subsection{ANS-GT + HDSE} 

We validate the performance of our HDSE framework using the efficient ANS-GT \cite{zhang2022hierarchical}, % further validate the performance of our framework on large-scale graphs by integrating the high-level HDSE into ANS-GT.
which uses a multi-armed bandit algorithm to adaptively sample nodes for attention. 
We use the Louvain method as our coarsening algorithm. And for each pair of nodes sampled adaptively by the ANS-GT, we calculate their HDSE and bias the attention computation. %The results of node classification are presented in Table \ref{tab:tab4}.
For fair comparisons, we tune the hyperparameters using the same grid search as reported in their paper \cite{zhang2022hierarchical}. Note that we report the supervised learning setting (different from the text), since this is the one considered in the ANS-GT \cite{zhang2022hierarchical}.
Overall, Table~\ref{tab:ansgt} shows that HDSE yields consistent performance improvements, even in this challenging scenario, where nodes are sampled. %This result demonstrates that our HDSE approach can effectively scale to large graphs and provide the right bias for graph transformers.

\begin{table*}[h]
    \centering
    
    \caption{\textbf{Node classification accuracy} on ANS-GT + HDSE (\%).
     The baseline results were taken from \citep{zhang2022hierarchical}. We apply 3 independent runs on random data splitting. $^+$ indicates the results obtained from our re-running.}
    % \vspace{-0.1 in}
    \begin{tabular}{c|ccccc}
        \toprule
        Model & Cora & Citeseer & Pubmed \\
         % \midrule %
        % &Accuracy $\uparrow$ & Accuracy $\uparrow$  & Accuracy $\uparrow$  & Accuracy $\uparrow$  & Accuracy $\uparrow$\\
        \midrule %
        GCN & 87.33 $\pm$ 0.38 & 79.43 $\pm$ 0.26 & 84.86 $\pm$ 0.19 \\
        GAT & 86.29 $\pm$ 0.53 & \textbf{80.13 $\pm$ 0.62} & 84.40 $\pm$ 0.05 \\
        % GraphSAGE & 86.90 $\pm$ 0.84 & 79.23 $\pm$ 0.53 & 86.19 $\pm$ 0.18 \\
        APPNP & 87.15 $\pm$ 0.43 & 79.33 $\pm$ 0.35 & 87.04 $\pm$ 0.17 \\
        JKNet & 87.70 $\pm$ 0.65 & 78.43 $\pm$ 0.31 & 87.64 $\pm$ 0.26 \\
        H2GCN & 87.92 $\pm$ 0.82 & 77.60 $\pm$ 0.76 & 89.55 $\pm$ 0.14 \\
        GPRGNN & 88.27 $\pm$ 0.40 & 78.46 $\pm$ 0.88 & 89.38 $\pm$ 0.43 \\
        \midrule %
        GT & 71.84 $\pm$ 0.62 & 67.38 $\pm$ 0.76 & 82.11 $\pm$ 0.39 \\
        SAN & 74.02 $\pm$ 1.01 & 70.64 $\pm$ 0.97 & 86.22 $\pm$ 0.43 \\
        Graphormer & 72.85 $\pm$ 0.76 & 66.21 $\pm$ 0.83 & 82.76 $\pm$ 0.24 \\
        Gophormer & 87.65 $\pm$ 0.20 & 76.43 $\pm$ 0.78 & 88.33 $\pm$ 0.44 \\
        Coarformer & 88.69 $\pm$ 0.82 & 79.20 $\pm$ 0.89 & 89.75 $\pm$ 0.31 \\
        \midrule %
        ANS-GT & 88.60 $\pm$ 0.45 & 77.75 $\pm$ 0.79$^+$ & 89.56 $\pm$ 0.55 \\
        \textbf{ANS-GT + HDSE}  & \textbf{89.67 $\pm$ 0.39} & 78.31 $\pm$ 0.58 & \textbf{90.63 $\pm$ 0.26} \\
        \bottomrule
    \end{tabular}
    \label{tab:ansgt}

\end{table*}

\subsection{Clustering Coefficients Analysis} 
We check if there is a correlation with the cluster structure according to \cite{hu2020open}, by computing clustering coefficients on five benchmarks from \cite{dwivedi2023benchmarking}, but we do not observe a direct correlation. Notably, the ZINC dataset, which comprises small molecules, has a low clustering coefficient; however, our HDSE shows a significant improvement on it. This improvement could be attributed to the HDSE capturing chemical motifs that cannot be captured by the clustering coefficient, as illustrated in Figure \ref{fig:HDSE}. 
\begin{table*}[h]
\scriptsize
	\centering
        \vspace{-0.1 in}
         \caption{Clustering Coefficients Analysis.}
        % \vspace{-0.1 in}
	\begin{tabular}{lccccc}
		\toprule
		{Model} & {ZINC} & {MNIST} & {CIFAR10} & {PATTERN} & {CLUSTER}\\
		% \cmidrule{2-3,4-5,6-7}
		& MAE $\downarrow$ & Accuracy $\uparrow$ & Accuracy $\uparrow$ & Accuracy $\uparrow$ & Accuracy $\uparrow$\\
             \midrule %
         Average Clust. Coeff. & 0.006 & 0.477 & 0.454 & 0.427 & 0.316\\
        \midrule %
        GT & 0.226 $\pm$ 0.014 &  90.831 $\pm$ 0.161  &  59.753 $\pm$ 0.293 & 84.808 $\pm$ 0.068  &  73.169 $\pm$ 0.622 \\
        \textbf{GT + HDSE}  & 0.159 $\pm$ 0.006 &  94.394 $\pm$ 0.177 &  64.651 $\pm$ 0.591 & 86.713 $\pm$ 0.049 &  74.223 $\pm$ 0.573   \\
        \midrule %
        SAT & 0.094 $\pm$ 0.008 &  – &  – & 86.848 $\pm$ 0.037 & 77.856 $\pm$ 0.104  \\
        \textbf{SAT + HDSE}  & 0.084 $\pm$ 0.003 &  – &  – & 86.933 $\pm$ 0.039 & {78.513 $\pm$ 0.097 } \\
        \midrule %
        GraphGPS & 0.070 $\pm$ 0.004 & 98.051 $\pm$ 0.126 & 72.298 $\pm$ 0.356 & 86.685 $\pm$ 0.059 & 78.016 $\pm$ 0.180 \\
        \textbf{GraphGPS + HDSE} & 0.062 $\pm$ 0.003 & 98.367 $\pm$ 0.106 & 76.180 $\pm$ 0.277 & 86.737 $\pm$ 0.055 & 78.498 $\pm$ 0.121 \\
        \bottomrule
	\end{tabular}
    \vspace{-0.1 in}
	\label{tab:tab2-ap}
\end{table*}

% \clearpage
\section{HDSE Visualization}
Here, we demonstrate that our HDSE method also provides interpretability compared to the classic GT. We train the GT + HDSE and GT on ZINC and Peptides-func graphs, and compare the attention scores between the selected node and other nodes.
Figure \ref{fig:attention} visualizes the attention scores on ZINC and Peptides-func. The results indicate that, after integrating the HDSE bias, the attention mechanism tends to focus on certain community structures rather than individual nodes as seen in classic attention. This phenomenon demonstrates our method's capacity to capture multi-level hierarchical structures. 

\begin{figure}[htb]  \center{\includegraphics[width=13cm]  {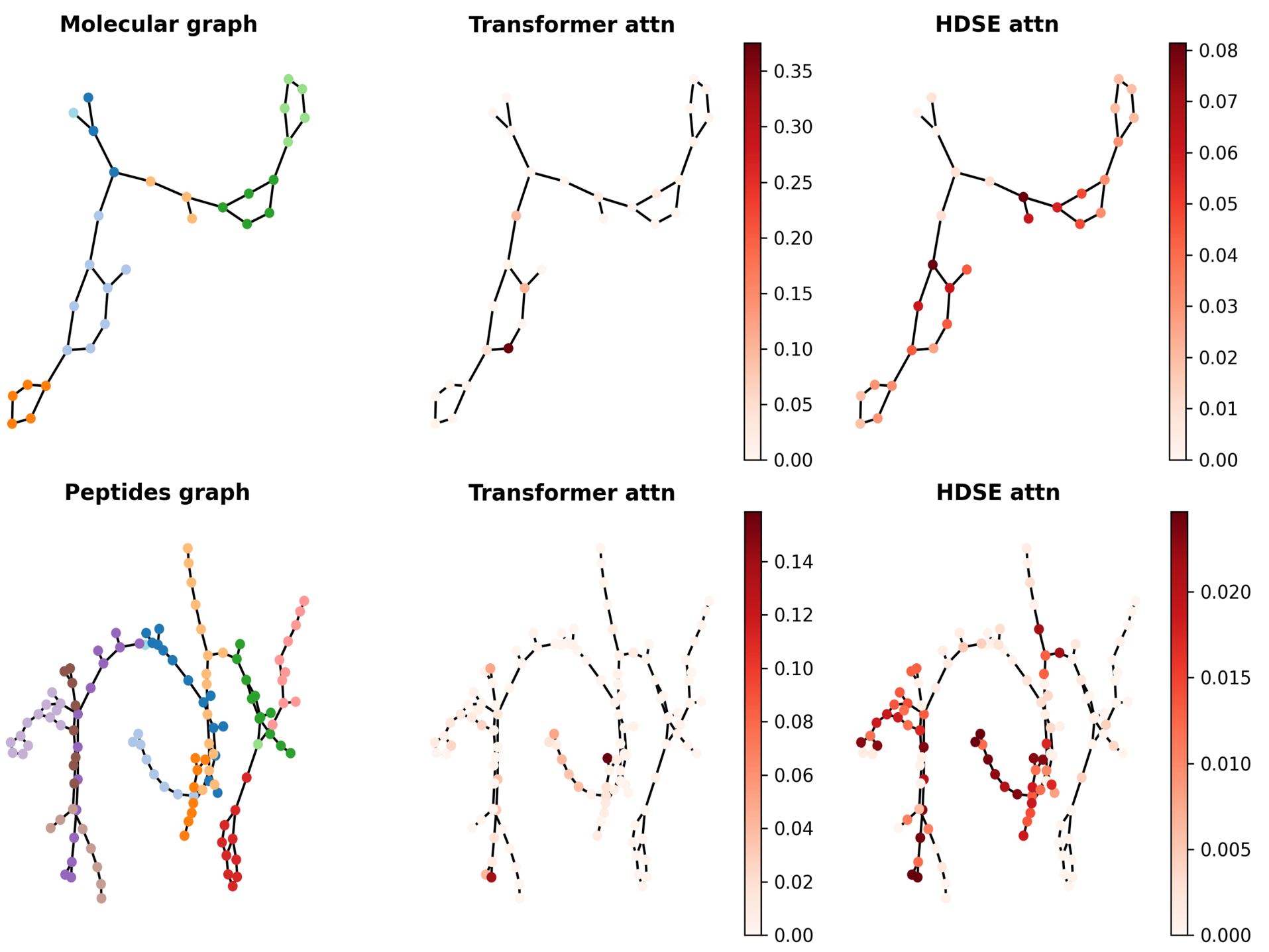}}   \caption{Visualization of attention weights for the transformer attention and HDSE attention. The left side illustrates the graph coarsening result. The center column displays the attention weights of a sample node learned by the classic GT \cite{dwivedi2020generalization}, while the right column showcases the attention weights learned by the HDSE attention.} \label{fig:attention} \end{figure}

\section{Additional Discussion} \label{ap-e}
\textbf{Positional Encoding or Structural Encoding?} 

We would like to clarify that our HDSE method aligns with the definitions of structural encoding. While GraphGPS \cite{rampavsek2022recipe} does classify SPD encoding under relative positional encoding, it defines relative structural encoding as "allow two nodes to understand how much their structures differ". Given that our HDSE not only incorporates SPD information but also embeds multi-level graph hierarchical structures, it is reasonable to classify it under the category of structural encoding. 

\textbf{Limitations.} In larger graphs, the presence of multi-level structures may require a higher maximal hierarchy level, $K$. The marginal improvements observed with increased $K$ may indicate limitations in the coarsening algorithm. We anticipate that the real potential of a higher $K$ will be revealed through the application of a proper, effective coarsening algorithm on graphs with hierarchical community structures. We look forward to exploring this in the future.

\textbf{Broader Impacts.} This paper presents work whose goal is to advance the field of Machine Learning. There are many potential societal consequences of our work, none which we feel must be specifically highlighted here.

% \newpage
\section{Further Related Works}
\textbf{Graph Transformers over Clustering Pooling.} \cite{pmlr-v202-kong23a} employs a hybrid approach that integrates a neighbor-sampling local module with a global module, the latter featuring a trainable, fixed-size codebook obtained by K-Means to represent global centroids, which is noted for its efficiency. Meanwhile, Gapformer \cite{liu2023gapformer} involves the incorporation of a graph pooling layer designed to refine the key and value matrices into pooled key and value vectors through graph pooling operations. This approach aims to minimize the presence of irrelevant nodes and reduce computational demands. However, the performance of these methods remains constrained due to a lack of effective inductive biases.

\textbf{Graph Transformers over Virtual Nodes.} Several graph transformer models utilize anchor nodes or virtual nodes for message propagation. For instance, Graphormer \cite{ying2021transformers} introduces a virtual node and establishes connections between the virtual node and each individual node. AGFormer \cite{jiang2023agformer} selects representative anchors and transforms node-to-node message passing into an anchor-to-anchor and anchor-to-node message passing process. Additionally, AGT \cite{ma2023rethinking} extracts structural patterns from subgraph views and designs an adaptive transformer block to dynamically integrate attention scores in a node-specific manner.

\end{document}